\documentclass[ journal, comsoc ]{IEEEtran}
\usepackage[T1]{fontenc}% optional T1 font encoding

\ifCLASSINFOpdf
\else
\fi
\usepackage{amsmath}
\usepackage{amsthm}
\usepackage[cmintegrals]{newtxmath}
\usepackage{amssymb}
\usepackage{bm}
\usepackage{bbm}
\usepackage{subfigure}
\usepackage{graphicx}
\usepackage{calc}
\usepackage{epstopdf}
\newtheorem{theorem}{Theorem}
\newtheorem{lemma}{Lemma}

\newtheorem{proposition}{Proposition}
\newtheorem{corollary}{Corollary}

\newtheorem{assumption}{Assumption}
\usepackage{verbatim}
\usepackage{enumitem}
\usepackage{color}
\usepackage{array}
\usepackage{cite}

\usepackage{xcolor}
\usepackage{floatrow}
\usepackage{lipsum}

\usepackage{fancyhdr}
\fancyheadoffset{0mm} 
\fancyfootoffset{0mm} 

\fancypagestyle{}{ 
	\tiny
	\fancyhf{} 
	\rhead{\thepage}} 
\usepackage[ruled,norelsize]{algorithm2e}
\usepackage{algorithmic}

\begin{document}
\title{  Mobility-aware Content Preference Learning in Decentralized Caching Networks}

\author{Yu~Ye,~\IEEEmembership{Student~Member,~IEEE,}
	Ming~Xiao,~\IEEEmembership{Senior~Member,~IEEE,}
	and~Mikael~Skoglund,~\IEEEmembership{Fellow,~IEEE }
	
	\thanks{ Yu Ye, Ming Xiao and Mikael Skoglund are with the School of Electrical Engineering and Computer Science, Royal Institute of Technology (KTH), Stockholm, Sweden (email: yu9@kth.se, mingx@kth.se, skoglund@kth.se).}% <-this % stops  
}
  
\maketitle

\begin{abstract}
Due to the drastic increase of mobile traffic, wireless caching is proposed to serve repeated requests for content download. To determine the caching scheme for decentralized caching networks, the content preference learning problem based on mobility prediction is studied. We first formulate preference prediction as a decentralized regularized multi-task learning (DRMTL) problem without considering the mobility of mobile terminals (MTs). The problem is solved by a hybrid Jacobian and Gauss-Seidel proximal multi-block alternating direction method (ADMM) based algorithm, which is proven to conditionally converge to the optimal solution with a rate $O(1/k)$. Then we use the tool of \textit{Markov renewal process} to predict the moving path and sojourn time for MTs, and integrate the mobility pattern with the DRMTL model by reweighting the training samples and introducing a transfer penalty in the objective. We solve the problem and prove that the developed algorithm has the same convergence property but with different conditions. Through simulation we show the convergence analysis on proposed algorithms. Our real trace driven experiments illustrate that the mobility-aware DRMTL model can provide a more accurate prediction on geography preference than DRMTL model. Besides, the hit ratio achieved by most popular proactive caching (MPC) policy with preference predicted by mobility-aware DRMTL outperforms the MPC with preference from DRMTL and random caching (RC) schemes. 	
	
%The proactive content caching scheme based on mobility prediction and content preference learning is studied for a two-tier which consists of a small base station (SBS) tier and a device-to-device (D2D) tier. 
%The discrete Markov renewal process is utilized to predict the moving paths and sojourn time for mobile terminals (MTs). Based on the predicted mobility patterns and true content preferences, we formulate the optimization problem of maximizing the hit rate for proactive caching. We provide the optimal solution for the scenario where popular contents are cached at SBSs only, and develop a greedy mobility prediction based proactive wireless caching (MPPC) scheme for the scenario where popular contents cached at both SBSs and MTs. We also show that the hit-rate achieved by MPPC is at least $\frac{\exp(1)-1}{\exp(1)}$ of the optimal hit-rate. Furthermore, by adopting the multitask learning (ML) method, we propose an ML based extreme learning machine (ML-ELM) to predict the content preference of MTs, which is proven to guarantee a smaller training error than that of single ELM learning approach for arbitrary training samples and hidden nodes. Finally, the real trace driven simulations show that the proposed MPPC with true content preference can improve the hit-rate by $35\%$ at most compared with most popular caching (MPC). Besides, the hit-rate achieved by MPPC with predicted content preference only degrades by $4\%$ at most.  

\end{abstract}

\IEEEpeerreviewmaketitle

\section{Introduction} 
%{ \color{red}{part 1: state the importance of MT mobility and content preference in proactive caching, and related works
%
%part 2: state the learning based method, and introduce the distributed learning methods in caching network
%
%part 3: own work and contribution
%}  }

As a promising technology for the fifth-generation (5G) wireless networks and beyond, \textit{proactive caching} can alleviate the heavy traffic burden on backhaul links and reduce service delay, through proactively storing popular contents at base stations (BSs) and mobile terminals (MTs) \cite{CMG14,BBD14,MSA16}. With the limitation of storage memory, determining where and what to cache in content centric wireless networks becomes one of the main challenges in the design of proactive caching schemes. Among the various factors affecting the wireless caching design, involving the mobility of MTs and learning content preference are two critical challenges, which have attracted more and more research interest recently. 
 
\subsection{background} 
 
Current investigation on mobility aware wireless caching mainly includes two aspects: studying the impact of MT mobility on caching schemes \cite{CMG16,GMM14,KD17,Wang2018TWC}, and optimizing the wireless caching schemes based on the mobility information of MTs \cite{PT13,Ozfa2018,TWCM17,WZS17,Zhang2019TVT2,Tan2018TVT,Hosny2016,Liu2017Acess,WCNCY17,Chen2017TWC}. A general framework on mobility-aware caching in content-centric wireless networks
is presented in \cite{CMG16}. In \cite{GMM14}, groups of mobile devices collaborating to exchange contents via BS assisted D2D communications, the deterministic and random caching strategies at MTs are analyzed, and it is shown that the latter may be more realistic in networks with MT mobility. The effect of user mobility on the coverage probability of D2D networks with distributed caching is studied in \cite{KD17}. 

Meanwhile, most of the recent results optimize caching schemes by considering mobility over different metrics. In \cite{PT13,Ozfa2018}, with the goal of minimizing the probability of using macro BSs for content delivery, mobility-aware storage allocation schemes for wireless caching in a two-tier heterogeneous network (HetNet) is studied. In \cite{TWCM17}, the proactive caching problem is investigated for cloud radio access networks (CRANs), where the caching scheme for remote radio heads and the baseband units (BBUs) is optimized to maximize the effective capacity. The mobility patterns of MTs are predicted through echo state networks (ESNs) at BBUs. In \cite{WZS17}, the inter-contact time between MTs is considered, and a mobility-aware caching strategy is developed to maximize the percentage of requested data delivered via D2D links. 
%The effectiveness of vehicular caching in content centric networks is studied in \cite{Zhang2019TVT1}, with the aim of minimizing energy consumption. 
In \cite{Liu2017Acess}, a novel mobility-aware coded probabilistic caching scheme is proposed for mobile edge computing (MEC) enabled small cell networks (SCNs) to maximize the throughput. The optimization of mobility-aware caching schemes for maximizing the hit ratio is studied in \cite{WCNCY17} and \cite{Chen2017TWC}.
The contact time of MTs is evaluated in \cite{WCNCY17} through modeling MT mobility as a \textit{Markov renewal process}, and the proactive caching design at BSs and MTs is optimized to maximize the hit ratio. While a peer-to-peer connectivity model is used to obtain the mobility pattern of users and mobility-aware caching placement for maximizing the cache hit ratio is investigated in \cite{Chen2017TWC}. 

For most references above, the content preference of MTs is assumed to be known. This however is not realistic. Since the preference of contents, which varies with time and location, plays an important role to the performance of proactive caching scheme, it should be continuously learned and updated. According to \cite{Bres1999}, content preference for a group of users follows a Zipf-like distribution. Hence learning methods are proposed to obtain the content preference from request history and the context of MTs \cite{TWCM17,Zhang2019TVT2,TCB16,Tanzil2017,Muller2017,Chen2018TC,Jiang2019}. In \cite{TWCM17}, the ESNs are also used at BBUs to predict MT content request distribution from the context information of MTs. But the implementation of the proposed method is complex since the system has to generate one ESN for each MT. Reference \cite{Zhang2019TVT2} provides a user interest prediction model, which combines the social proximity and dynamic content popularity. In \cite{TCB16}, content popularity is unknown and estimation uses instantaneous demands from users within a specified time interval. To reduce learning time, a transfer learning-based approach is proposed, where extra source domain samples are provided for estimation. The authors in \cite{Tanzil2017} construct an extreme learning machine (ELM) to estimate popularity. In \cite{Song2017}, content preference is learned by multi-armed bandit method incorporated with the content caching and sharing process. The algorithm that learns context-specific content popularity online by regularly observing the context information of connected users is presented in \cite{Muller2017}. Further in \cite{Jiang2019}, an offline user preference learning approach is studied.

\subsection{motivation} 

To reduce traffic, wireless caching seeks to store the common popular contents at the edge of the network. The geographic caching scheme is studied in \cite{Bla2015ICC}. Since the MTs may have heterogeneous preference \cite{Guo2017TC}, the mobility of MTs will cause the content preference in different locations to vary with time. This is presented in Fig. \ref{fig1}, where the geography preference observed by agent 4 is determined by the movement of MTs. Thus to achieve mobility-aware content placement, wireless caching networks should learn the geography preference incorporating the mobility patterns of MTs rather than only learn the individual MT preference \cite{TCB16,Tanzil2017,Muller2017,Chen2018TC,Jiang2019}.   
 
\begin{figure}[h]    
	\begin{center}
		\includegraphics[width=80 mm]{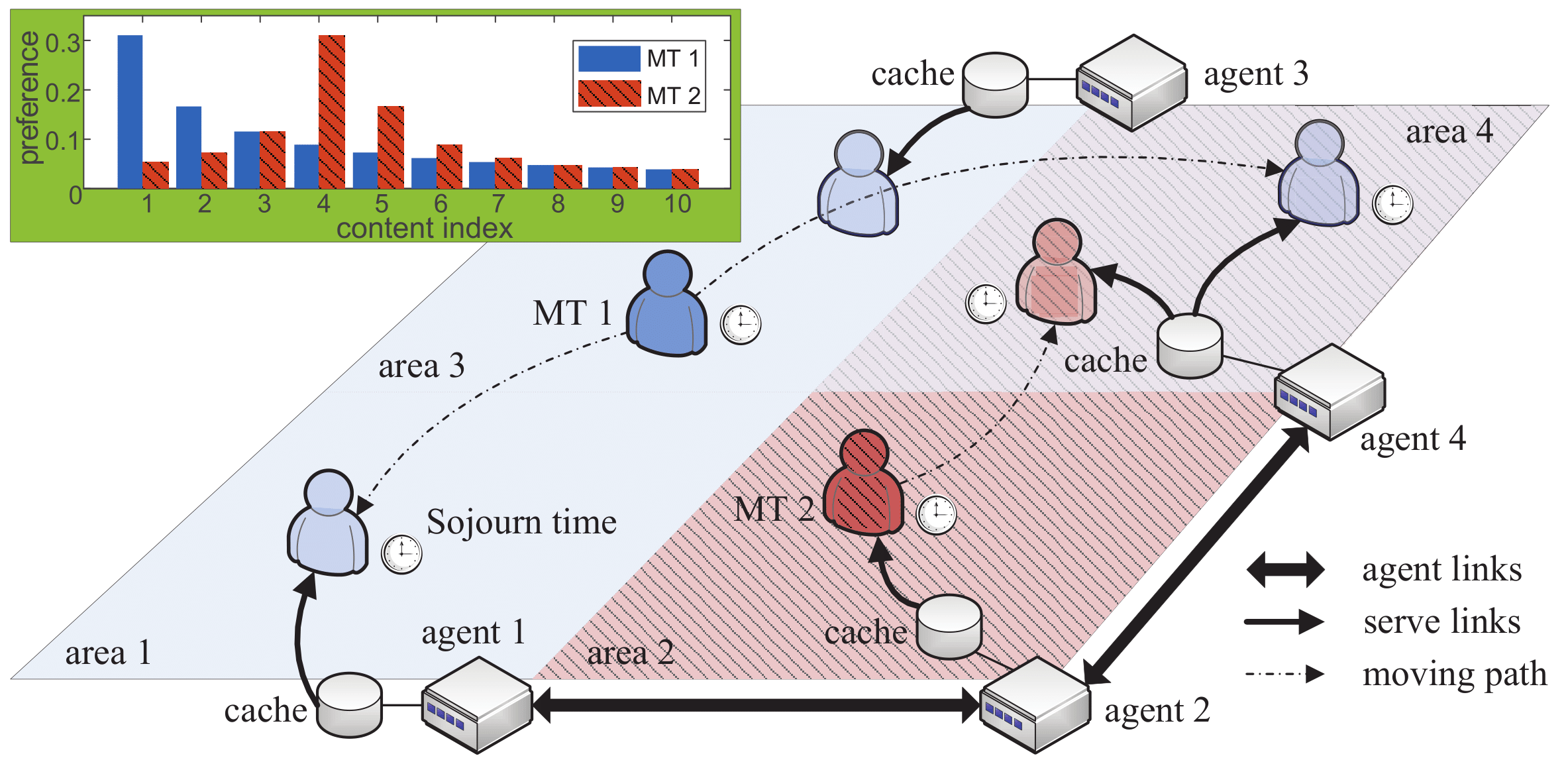}
		\caption{The system model of mobility-aware MTL proactive caching in decentralized content centric wireless networks.}
		\label{fig1}
	\end{center}
\end{figure}  
 
%1. why consider distributed setting
Since the data used to learn the user preference is first observed by local agents, in order to implement the learning in a central manner, tremendous amounts of data need to be transmitted. This can cause heavy communication load. Furthermore, due to the constraint of data privacy and security, distributed learning method is preferred on predicting the geography preference in decentralized networks. Although distributed caching is studied in \cite{Borst2010,Marini2015,Liu2016ICC}, none of these references provides machine learning based approaches for predicting the mobility-aware geography preference. 
%2. why using multi-task learning 
Considering cooperative caching \cite{Yu2017ICC} and transmission \cite{Chen2017TWC1}, the proactive caching schemes in different spatial areas should be optimized jointly. Hence the agent in Fig. \ref{fig1} should be aware of what the preference is like at its neighboring areas. To achieve this requirement, we take advantage of the regularized multi-task learning (RMTL) method \cite{Evgeniou2004} and extend it to a decentralized setting. This is because the geography preference for adjacent area is correlated due to the movement of MTs.

In what follows, we shall study how to predict the content preference in geography aspects with distributed learning methods and mobility prediction. 
The main contributions of this paper are listed as follows.
 
\begin{itemize}
	\item We model geography preference prediction in distributed caching networks as a decentralized RMTL (DRMTL) problem, which is solved by proposed hybrid Jacobian and Gauss-Seidel proximal multi-block alternating direction method (ADMM).
	
	\item We model the MT moving pattern through a \textit{Markov renewal process} to predict the moving paths and sojourn time. Then we integrate the mobility into the DRMTL model by reweighting examples and introducing a transfer penalty in order to control the information exchanged across adjacent agents. The problem solution is provided.

	\item We generalize the DRMTL problem as majorized multi-block convex optimization with coupled objective functions. We show that the mobility-aware DRMTL model is consistent with the DRMTL problem, and prove that proposed solutions converge to optimum with $O(1/k)$ rate when algorithm parameters meet specific conditions.
	
% 	\item We model the MT moving pattern in UDN through a discrete \textit{Markov renew process}, which is used to predict the moving paths and sojourn time for MTs. Moreover we use the model to characterize the contact time among MTs, which is used to evaluate the hit-rate achieved by D2D communication.
 	
% 	\item We present the proactive caching scheme for SBSs and MTs in UDN by formulating the maximum hit-rate problem. The optimal proactive caching scheme for SBSs is first provided in the scenario where only SBS caching is available. Then we develop a heuristic algorithm called mobility prediction based proactive caching (MPPC) to solve the maximum hit-rate problem for the scenario of caching at both SBSs and MTs. The algorithm complexity and performance of MPPC are analyzed.
 	 
% 	\item We propose a multitask learning based extreme learning machine (ML-ELM) to estimate the content preference of MTs. Through theoretical analysis we show that the proposed method can outperform the single ELM learning approach with respect to the training error.	
 	\item By simulation, we verify the convergence of proposed algorithms. With experiments on real trajectory dataset, we show that, compared with the DRMTL model, the mobility-aware DRMTL model can provide a more accurate prediction on geographic preference, with which the most popular proactive caching (MPC) scheme can achieve a better hit ratio than MPC with preference from DRMTL and random caching (RC) scheme.
 	
% 	\item By simulation results, we show that the hit-rate achieved by the proposed MPPC scheme outperform that of existing MPC and RC schemes. Moreover, the hit-rate of MPPC with predicted content preference only has $4\%$ reduction from the scenario with true preference.
\end{itemize}

The rest of this paper is organized as follows. We present the system model of preference prediction and formulate it as a DRMTL problem in Section II. Then hybrid Jacobian and Gauss-Seidel proximal ADMM based algorithm (Algorithm 1) is provided, together with its convergence analysis. In section III the mobility prediction model is first presented, followed by the formulation of mobility-aware DRMTL problem, as well as the proposed Algorithm 2 and convergence proof. Numerical results are given in Section IV to evaluate the convergence and preference prediction performance of the proposed approaches. Finally, we draw conclusions in Section V.

\section{System Model and Decentralized Content Preference Learning}

The system model for content preference prediction is shown in Fig. \ref{fig1}. Each spatial area is controlled by one agent, and there is no overlap for adjacent areas. In the DRMTL setup, each task runs at one agent. Hence for simplicity, we use the agent to stand for the corresponding area and task. We define the decentralized network by an undirected graph $\mathcal{G}(\mathcal{V},\mathcal{E})$, where $\mathcal{V}$ includes all agents and $\mathcal{E}$ includes all connection among them. We denote $\mathcal{V}_i\subset\mathcal{V}$ as the neighboring agents for $i\in\mathcal{V}$, where $(i,j)\in\mathcal{E}$ if $j\in\mathcal{V}_i$. The MTs located within areas $\mathcal{V}$ are $\mathcal{M}=\{1,...,M\}$, and the interested contents for MTs are $\mathcal{F}=\{1,...,F\}$, where $F\in\mathbbm{N}_{+}$. Following \cite{Guo2017TC} we categorize MTs into $K$ groups according to their content preference, denoted by the set $\mathcal{K}=\{1,...,K\}$. The content preference for group $i\in\mathcal{K}$ is $\text{P}_i$. Without loss of generality, we assume MTs requesting contents with fixed rate $R_f$. For each request from an MT in group $i\in\mathcal{K}$, it follows distribution $\text{P}_i$. Moreover the request can only be served by the agent where the MT located. In order to determine the content placement in proactive caching for future time window $[0,t_d]$, each agent needs to predict the content preference from the observed request history $\mathcal{D}_i=\{x_{i,l},y_{i,l}|l=1,...,b_i \}$ before time $0$ at its covered area. Then each agent determines what to cache at time $0$. During $[0,t_d]$, MTs may move around within agents. Besides we assume that dataset cannot be shared among neighboring agents.
%, but only the intermediate results during the training process.  
 
In what follows, we will first formulate the decentralized content preference learning model based on DRMTL.

%Then a hybrid Jacobi and Gauss-Seidel proximal ADMM is provided to obtain the optimal model parameters, followed by detailed convergence analysis.  

%we will first formulate proactive content placement to maximize the hit-rate of UDN as an optimization problem based on the mobility prediction of MTs. Then the optimal and heuristic solution for the scenario, where caching at SBSs and caching at both SBSs and MTs, will be developed respectively. 

\subsection{Problem Formulation}
Considering a decentralized network $\mathcal{G}(\mathcal{V},\mathcal{E})$, for simplicity, we denote $|\mathcal{V}|=N$.
We formulate preference prediction by the DRMTL model, where MTs are assumed to be stationary and agents cooperatively solve the following problem 
 \begin{equation}
 	\min_{ \bm{w}_i}~\sum\nolimits_{i\in\mathcal{V}}f_i\big(\bm{w}_i \big),
 \end{equation} 
where $\bm{w}_i =\bm{w}_0 + \hat{\bm{w}}_i$ and $f_i(\cdot)$ is the local loss function of task $i$.
$\bm{w}_0$ is common among agents while $\hat{\bm{w}}_i$ is the specific weight for tasks. The main reason of using DRMTL in content preference learning is that $\bm{w}_0$ can capture basis content preference across the adjacent geography areas, while $\bm{w}_i$ represents the variance of content preference of task $i$.
% Hence having the knowledge about $\bm{w}_0$ is helpful to determine the caching strategy in higher level of cache unit such as the structure studied in \cite{Yu2017ICC}, while $\bm{w}_i$ helps to determine what to cache locally. 
 Similar to \cite{Evgeniou2004}, the model parameters $\bm{w}_0$ and $\hat{\bm{w}}_i$ can be learned at the central processor, if it has access to all datasets. However, the data collected by different agents are geo-distributed. Therefore we consider the RMTL in decentralized networks. In the decentralized setting, if the parameters are learned separately by agents, a common $\bm{w}_0$ cannot be guaranteed. To solve this problem, we exploit the alternating direction method of multipliers (ADMM) to ensure all agents agree with the same basis by introducing basis weights $\check{\bm{w}}_i$ at each agent and the consensus constraints $\check{\bm{w}}_i=\check{\bm{w}}_j,\forall (i,j)\in\mathcal{E} $. The problem is formulated as

 \begin{equation}
	\begin{aligned}
	 \min_{\check{\bm{W}} ,\hat{\bm{W}}  } &~\sum\nolimits_{i\in\mathcal{V}}\left(\vphantom{\frac{1}{2}} f_i \big(\check{\bm{w}}_i,\hat{\bm{w}}_i,\mathcal{D}_i \big) + \frac{\mu_1}{2} \big\|\check{\bm{w}}_i  \big\|^2+ \frac{\mu_2}{2} \big\|\hat{\bm{w}}_i  \big\|^2 \right),\\
	 s.t.&~ \bm{A} \check{\bm{W}}=\bm{0}.
	\end{aligned}      \tag{P1}
 \end{equation} 
 where $\check{\bm{W}}=[\check{\bm{w}}_1,...,\check{\bm{w}}_N]$, $\hat{\bm{W}}=[\hat{\bm{w}}_1,...,\hat{\bm{w}}_N]$, $\bm{A}=[A_1,...,A_N]$, and $\bm{A}$ can be derived by $\mathcal{G}$. $\mathcal{D}_i $ is the dataset at agent $i$, where $x_{i,l}\in\mathbbm{R}_{+}^n$ is the feature of MT while $y_{i,l}\in\mathbbm{R}_{+}^\nu$ is the associated request history. Without loss of generality, we assume $\check{\bm{w}}_i,\hat{\bm{w}}_i\in\mathbbm{R}^n$, and denote $f_i(\check{\bm{w}}_i,\hat{\bm{w}}_i,\mathcal{D}_i)=\frac{1}{b_i}\sum\nolimits_{l=1}^{b_i}\ell(\check{\bm{w}}_i,\hat{\bm{w}}_i,x_{i,l},y_{i,l})$ as the local function, where $\ell(\cdot):\mathbbm{R}^n\to \mathbbm{R}^v$ is the loss function. The predicted preference during $[0,t_d]$ at agent $i$ can be obtained by $y_{i,0}=f_i(\check{\bm{w}}_i,\hat{\bm{w}}_i,x_{i,0})$, where $x_{i,0}$ is the input at time $0$.  
For convenience we denote $f_i(\check{\bm{w}}_i,\hat{\bm{w}}_i,\mathcal{D}_i)$ as $f_i(\check{\bm{w}}_i,\hat{\bm{w}}_i)$.

 \begin{algorithm}[h]
	\caption{} 
	\begin{algorithmic}[1]
		\STATE \textbf{initialize}: set $\{\check{\bm{w}}^0_i,\hat{\bm{w}}_i^0 |i\in\mathcal{V}\} $ randomly and $\lambda^0=\bm{0}$  	
		\FOR{$k=0,1,...$ } 
		\STATE {\bfseries agents $i=1$ to $N$}:
		\STATE update $\check{\bm{w}}^{k+1}_i$ with (\ref{eq3}) \textit{in parallel} ; 
		\STATE communicate $\check{\bm{w}}_i^{k+1}$ with neighboring agents and update $\lambda^{k+1}$ with (\ref{eq4}).	 
		\STATE {\bfseries agents $i=1$ to $N$}:
		\STATE update $\hat{\bm{w}}^{k+1}_i$ with (\ref{eq5}) \textit{in parallel}; 
		\ENDFOR 
	\end{algorithmic} 
\end{algorithm} 
 
Problem (P1) can be solved with the proximal ADMM method \cite{Deng2017}. The augmented Lagrange function is given by
 \begin{equation}
 	\begin{aligned}
 	\vphantom{\frac{1}{2}}\mathcal{L}_{\rho}^1 \big(\check{\bm{W}},\hat{\bm{W}},\lambda \big)=&\sum\nolimits_{i }\left(\vphantom{\frac{1}{2}} f_i \big(\check{\bm{w}}_i,\hat{\bm{w}}_i \big) + \frac{\mu_1}{2} \big\|\check{\bm{w}}_i  \big\|^2+\right.\\&~~~~~~~\left.\vphantom{\frac{1}{2}} \frac{\mu_2}{2} \big\|\hat{\bm{w}}_i  \big\|^2+ \lambda^T\bm{A}\check{\bm{W}} + \frac{\rho}{2} \big\| \bm{A}\check{\bm{W}}  \big\|^2\right),
 	\end{aligned}
 \end{equation}
where $\lambda$ is a Lagrange multiplier and $\rho~(>0)$ is a constant parameter.  
The constraint matrix satisfies $A_i^TA_i = d_iI$, where $I$ is an identity matrix, $d_i=|\mathcal{V}_i|$ is the degree of agent $i$, and $A_i^TA_j = \bm{0}$ if $(i,j)\notin \mathcal{E}$. By defining the set $\check{\bm{W}}_{-i}^k=[\check{\bm{w}}_1,...,\check{\bm{w}}_{i-1},\check{\bm{w}}_{i+1},...,\check{\bm{w}}_N]$, the update of hybrid Jacobi and Gauss-Seidel type ADMM follows 
\begin{align}
\check{\bm{w}}_i^{k+1} :=&\arg\min \mathcal{L}_{\rho}  \big(\check{\bm{w}}_i,\hat{\bm{w}}_i^k,\check{\bm{W}}_{-i}^k,\lambda^k  \big)+\frac{1}{2} \big\| \check{\bm{w}}_i-\check{\bm{w}}^k_i  \big\|^2_{P_i},\label{eq3}\\
\vphantom{\frac{1}{2}} \lambda^{k+1} :=& \lambda^{k} + \gamma \rho \bm{A}\check{\bm{W}}^{k+1},\label{eq4}\\
\hat{\bm{w}}_i^{k+1} :=&\arg\min \mathcal{L}_{\rho} \big(\hat{\bm{w}}_i,\check{\bm{w}}_i^{k+1}  \big)+\frac{1}{2} \big\| \hat{\bm{w}}_i-\hat{\bm{w}}^k_i  \big\|^2_{Q_i},\label{eq5}
\end{align}
where $k$ is the iteration number.
Then we conclude Algorithm 1 to solve (P1). In Algorithm 1, the values of $\check{\bm{w}}_i$ and $\hat{\bm{w}}_i$ are randomly initialized. In iteration $k$, $\check{\bm{w}}_i^{k+1}$ is firstly updated at each agent in parallel and then communicated to agent $j$ where $(i,j)\in\mathcal{E}$. The Lagrange $\lambda$ regarding to all connections is then updated with $\check{\bm{w}}_i^{k+1}$. At the end of iteration $k$, the local weights $\hat{\bm{w}}_i$ are optimized separately across agents. It is worth noting that Algorithm 1 is synchronous ADMM since the clock $k$ is kept unique among agents.  

% \begin{algorithm} [t]
% 	\SetAlgoLined
% 	\KwIn{ $\{\mathcal{D}_i,\tau_i,\zeta_i|i\in\mathcal{V}\},\rho,\gamma,\mu_1,\mu_2$,} 
% 	%	\KwIn{Data $\{D_t|t=1,\cdots m\}$}
% 	\KwOut{$\{\check{\bm{w}}_i,\hat{\bm{w}}_i|i\in\mathcal{V}\}$} 
% 	Initialization $ \{\check{\bm{w}}^0_i,\hat{\bm{w}}_i^0 |i\in\mathcal{V} \} $ randomly and $\lambda^0=\bm{0}$\; 	
% 	\For{$k=0,1,...$   }{ 
% 		{\bfseries agents $i=1$ to $N$}:\\
% 		 Update $\check{\bm{w}}^{k+1}_i$ \textit{in parallel} according to (\ref{})\; 		 	
% 		 Communicate $\check{\bm{w}}_i^{k+1}$ with neighboring agents and update $\lambda^{k+1}$ according to (\ref{ }).\\		  
% 		 {\bfseries agents $i=1$ to $N$}:\\   
% 		 Update $\hat{\bm{w}}^{k+1}_i$ \textit{in parallel} according to (\ref{ })\; 
% 		 
% 	} 	
% 	\caption{}
% \end{algorithm}

\subsection{Convergence analysis}
 In what follows, we will analyze the convergence properties of the proposed Algorithm 1. It is worth noting that (P1) is a majorized multi-block ADMM with coupled objective functions, which is preliminarily studied in \cite{Cui2016CPM,Gao2017,Xu2018SIAM}. Our proof is different from that of 
%		A more generalized presentation of P1 is the majorized model, which tries to solve the following problem
%\begin{equation}\label{eq6}
%\begin{aligned}
%\min_{\check{\bm{W}},\hat{\bm{W}} } ~ \phi \big(\check{\bm{W}},\hat{\bm{W}} \big) +  \sum\nolimits_{i} \varphi_i \big(\check{\bm{w}}_i\big) + \sum\nolimits_{i}\psi_i\big(\hat{\bm{w}}_i\big), ~s.t.~  \bm{A}  \check{\bm{W}} + \bm{B} \hat{\bm{W}}  =\bm{c}.\\
%% ~s.t.~\bm{A}\bm{U}+\bm{B}\bm{A}=\bm{b}. \label{eq6}
%%s.t.  ~& \sum\nolimits_{i}\left( A_i \check{\bm{w}}_i + B_i\hat{\bm{w}}_i\right)=\bm{c},
%\end{aligned}	
%\end{equation}
%From Multi-block ADMM algorithm in \cite{Gao2017}, problem (\ref{eq6}) can be solved by the BCD method with Gauss-Seidel type, which optimizes the variable sequentially while fixing the remaining blocks at their last updated values. However, the sequential update behavior is non-efficient for solving the decentralized optimization problem. Since the fully Gauss-Seidel update usually performs better than the fully Jacobian update empirically \cite{Xu2018SIAM}, we integrate Jacobi-Proximal ADMM \cite{Deng2017} with Gauss-Seidel update in Algorithm 1 to solve problem P1 in a hybrid way.
\cite{Xu2018SIAM}, where a hybrid Jacobian and Gauss-Seidel proximal block coordinate update (BCU) method is presented to solve a linearly constrained multi-block structured problem with a quadratic term in objective. 
%Furthermore, we will present a brief proof for the convergence by extending the results in \cite{Cui2016CPM}. 
The similar type algorithm is also provided in \cite{yu2019} to solve a multi-convex problem. We start with the following two assumptions.
\begin{assumption}
 The undirected graph $\mathcal{G}$ is connected.
\end{assumption}
The Assumption 1 ensures that the consensus for $\check{\bm{w}}_i$ can be guaranteed in Algorithm 1. Defining $\bm{z}_i = [\check{\bm{w}}_i, \hat{\bm{w}}_i]$, we present the assumption on the loss function $f_i$. 
\begin{assumption}
	The local loss function $f_i$ is differentiable and jointly convex over $\check{\bm{w}}_i$ and $\hat{\bm{w}}_i$. The gradient of $f_i$ is Lipschitz continuous  
	\begin{equation}
		 \big\|\nabla f_i \big(\bm{z}_i^1\big)- \nabla f_i \big(\bm{z}_i^2 \big)  \big\| \leq C_i \big\| \bm{z}_i^1 - \bm{z}_i^2  \big\| ,~\forall \bm{z}_i^1,\bm{z}_i^2\in \mathbbm{R}^{n}\times\mathbbm{R}^{n},\label{eq7}
	\end{equation}
	where $C_i$ is the Lipschitz constant.
\end{assumption}

Denoting $g(\check{\bm{w}}_i)=\frac{\mu_1}{2} \|\check{\bm{w}}_i \|^2$ and $h(\hat{\bm{w}}_i)=\frac{\mu_1}{2} \|\hat{\bm{w}}_i \|^2$, further with Assumption 2, we obtain the following useful inequality for loss function $f_i$. 
\begin{lemma}
	For any $ \bm{z}^1_i,\bm{z}^2_i,\bm{z}^3_i\in \mathbbm{R}^{n}\times\mathbbm{R}^{n}$, 
	\begin{equation}
	f_i \big(\bm{z}_i^2 \big)\leq f_i \big(\bm{z}_i^1 \big) +  \big(\bm{z}_i^2-\bm{z}_i^1 \big)^T \nabla f_i \big(\bm{z}_i^3 \big) + \frac{C_i}{2} \big\|\bm{z}_i^2-\bm{z}_i^3  \big\|^2. \label{eq8} 
	\end{equation}
	With the convexity of $g\big(\check{\bm{w}}_i\big)$ and the strong convexity of $h\big(\hat{\bm{w}}_i\big)$ with constant $ m(>0)$, we have
%	\begin{equation}\label{eq:3}
%	\big(\bm{z}_i^1-Z_t^2 \big)^T\nabla F_t \big(Z_t^2 \big)\leq F_t \big(Z_t^1 \big) - F_t \big(Z_t^2 \big),
%	\end{equation}
	\begin{align} 
	&\vphantom{\frac{1}{2}} \big(\check{\bm{w}}_i^1-\check{\bm{w}}_i^2 \big)^T g'  \big(\check{\bm{w}}_i^2 \big)\leq g'  \big(\check{\bm{w}}_i^1 \big) - g'  \big(\check{\bm{w}}_i^2 \big),\label{eq9}\\
    &\big(\hat{\bm{w}}_i^1-\hat{\bm{w}}_i^2 \big)^T h'  \big(\hat{\bm{w}}_i^2 \big)+\frac{m}{2} \big\| \hat{\bm{w}}_i^1-\hat{\bm{w}}_i^2  \big\|^2\leq h'  \big(\hat{\bm{w}}_i^1 \big) - h'  \big(\hat{\bm{w}}_i^2 \big).\label{eq10}
	\end{align} 
\end{lemma}
\begin{proof}
	Following Fact 2 in \cite{Drori2015}, (\ref{eq8}) can be shown directly. (\ref{eq9}) holds because of the convexity of $g(\cdot)$. Since $m=2$ can make (\ref{eq10}) satisfied, $h(\cdot)$ is strongly convex. 
\end{proof}
With Assumption 1 and Lemma 1, we then present the global convergence of Algorithm 1. To simplify the
notation, we denote
\begin{align}
\bm{G}_1&:=\text{blkdiag}\big(\rho A_1^TA_1 + P_1,..., \rho A_N^TA_N + P_N\big),\\
\vphantom{\frac{1}{2}} \bm{G}_2&:=\text{blkdiag}\big(Q_1,...,Q_N \big),\\
\vphantom{\frac{1}{2}} \bm{G}_3&:=\text{blkdiag}\left( \frac{C_1}{m}\big(C_1+m\big),...,\frac{C_N}{m}\big(C_N+m\big) \right), \\
\vphantom{\frac{1}{2}} \bm{G}  &:=\text{blkdiag}\left(\bm{G}_1,\bm{G}_2,\frac{1}{\gamma\rho}I \right), \\
\bm{M} &:= \begin{bmatrix}
\bm{G}_1 & \bm{0} &  \frac{1}{\gamma}  \bm{A} ^T\\
\bm{0}&\bm{G}_2-\bm{G}_{3}  &\bm{0} \\
\frac{1}{\gamma} \bm{A} &  \bm{0} & \frac{2-\gamma}{\gamma^2 \rho}I
\end{bmatrix}, 
\end{align}
where $\text{blkdiag}(\cdot)$ stands for the block-diagonal matrix.
%Define $\bm{u}=[\check{\bm{W}},\hat{\bm{W}},\lambda]$ and $\bm{u}^*=[\check{\bm{W}}^*,\hat{\bm{W}}^*,\lambda^*]$ as the optimal solution for P1. 
Denote $  \{ \bm{u}^{k+1} =  [\check{\bm{W}}^{(k+1)T},\hat{\bm{W}}^{(k+1)T},\lambda^{(k+1)T}  ]^T,k\geq 1 \}$ as the sequence generated by Algorithm 1 after the $k$-th iteration. Our analysis focuses on bounding the error $  \|\bm{u}^k-\bm{u}^*  \|^2_G$ and showing that it decreases with iterations, where $\bm{u}^*$ is the optimal solution for (P1).

\begin{lemma}
	For $k\geq 1$, the sequence $\bm{u}^k$ satisfies
	\begin{equation}
\big\|\bm{u}^k-\bm{u}^*  \big\|^2_{\bm{G}} -  \big\|\bm{u}^{k+1}-\bm{u}^* \big\|^2_{\bm{G}} \geq  \big\|\bm{u}^k-\bm{u}^{k+1} \big\|^2_{\bm{M}},
	\end{equation}
	where
	\begin{equation}
	\begin{aligned}
	&\vphantom{\frac{1}{2}}\big\| \bm{u}^k-\bm{u}^{k+1} \big\|^2_{\bm{M}}= \big\|\check{\bm{W}}^{k}-\check{\bm{W}}^{k+1} \big\|^2_{\bm{G}_1} + \big\|\hat{\bm{W}}^{k}-\hat{\bm{W}}^{k+1} \big\|^2_{\bm{G}_2-\bm{G}_3}+\\&  \frac{2-\gamma}{\gamma^2\rho} \big\| \lambda^k-\lambda^{k+1} \big\|^2 + \frac{2}{\gamma}  \big( \lambda^k-\lambda^{k+1}\big)^T \bm{A}  \big(\check{\bm{W}}^k-\check{\bm{W}}^{k+1} \big)    .
	\end{aligned}
	\end{equation}
\end{lemma}
\begin{proof}
	See the Appendix A.
\end{proof}

If matrix $\bm{M}$ is positive definite, then there exists some $\eta>0$ such that
\begin{equation}
 \big\| \bm{u}^k-\bm{u}^*  \big\|_{\bm{G}}^2-  \big\|\bm{u}^{k+1}-\bm{u}^*  \big\|^2_{\bm{G}}\geq \big\|\bm{u}^k-\bm{u}^{k+1}   \big\|_{\bm{M}}^2  \geq \eta   \big\|\bm{u}^{k}-\bm{u}^{k+1}   \big\|^2. \label{eq18}
\end{equation}
 (\ref{eq18}) shows that with increasing iteration $k$, error $ \| \bm{u}^k-\bm{u}^*   \|_{\bm{G}}^2$ is monotonically non-increasing and thus converging, and $ \|\bm{u}^k-\bm{u}^{k+1}   \| \to 0$. Then from the standard analysis for contraction methods \cite{he1997}, $ \| \bm{u}^k-\bm{u}^*   \|_{\bm{G}}^2\to 0$ can be obtained immediately. In the following, we provide the conditions that guarantee the global convergence of Algorithm 1.
 For convenience we adopt \textit{Standard Proximal} as $P_i=\tau_iI $ and $Q_i=\zeta_iI$, where $\tau_i, \zeta_i\in \mathbbm{R}_{+}$. 
\begin{theorem}(Convergence of Algorithm 1)
	If there exist $0< \epsilon_i <1$ and  $0<\gamma<2$ such that $\rho$, $\tau_i$ and $\zeta_i$ satisfy the following conditions:  
\begin{equation}
\tau_i>\rho\left(\frac{1}{\epsilon_i}-1 \right) d_{i },~\zeta_i>\frac{C_i}{m}(C_i+m),~\sum\nolimits_{i}\epsilon_i<2-\gamma,\label{eq19}
\end{equation}    
	then the sequence $\bm{u}^k$ generated by Algorithm 1 converges to the global optimal solution $\bm{u}^*$ of Problem (P1).
\end{theorem}
\begin{proof}
 See the Appendix B.
\end{proof}
 
\begin{proposition}
 By letting $\epsilon_i<\frac{2-\gamma}{N}$, condition (\ref{eq19}) reduces to 
\begin{equation}
	\tau_i > \rho \left(\frac{N}{2-\gamma}-1 \right) d_{i },~   
	\zeta_i > \frac{C_i}{m}(C_i+m).\label{eq20}
\end{equation}
   
\end{proposition}
\begin{proof}
	The results can be straightly obtained from the proof of Theorem 1.
\end{proof}
Next we shall investigate the convergence rate of Algorithm 1. Here, we define 
\begin{equation}
\bm{G}_1^{\dagger}:= \bm{G}_1-\rho\bm{A}^T\bm{A}  , 
\end{equation}
 and denote $\bm{F}(\bm{Z})=\sum\nolimits_{i}F_i(\bm{z}_i)=\sum_{i}(f_i(\check{\bm{w}}_i,\hat{\bm{w}}_i) + g(\check{\bm{w}}_i) +h(\hat{\bm{w}}_i) )$ with $\bm{Z}=[\check{\bm{W}},\hat{\bm{W}}]$.
 \begin{corollary}(Convergence rate of Algorithm 1)
 If $\bm{G}_1^{\dagger}\succ 0$, $\bm{G}_2 \succ \bm{G}_3$, $ 0<\gamma<2 $ and $\lambda^0=\bm{0}$, then we have
 \begin{equation}
 \begin{aligned}
  \bm{F} \big(\overline{\bm{Z}}^k \big)- \bm{F}  \big(\bm{Z} ^* \big)  \leq  \frac{1}{2k}\bigg(  \big\|  \check{\bm{W}}^0-\check{\bm{W}}^* \big\|^2_{\bm{G}_1^{\dagger}} +  \big\|  \hat{\bm{W}}^0-\hat{\bm{W}}^* \big\|^2_{\bm{G}_2 } \bigg),\label{eq22}
 \end{aligned}
 \end{equation}
 where $\overline{\bm{Z}} ^k=\frac{1}{k}\sum_{i=1}^{k}\bm{Z} ^i$.
\end{corollary}
\begin{proof}
See the Appendix C.
\end{proof} 
Corollary 1 demonstrates that Algorithm 1 can guarantee a $O(1/k)$ convergence rate in average.  
 
 \begin{proposition}
 	Since $  \|A_i^TA_j   \| \leq \sqrt{n} ~(i,j\in\mathcal{V},i\neq j)$, to meet the requirements in Theorem 1 and Corollary 1, $\zeta_i$ should satisfy (\ref{eq20}) while $\tau_i$ follows
 \begin{equation}
 \begin{aligned}
 \tau_i  >\max  \left\{\rho d_i+4(N-1)\rho\sqrt{n},~\rho  \left(\frac{N}{2-\gamma}-1 \right) d_i  \right\}.\label{eq23}
 \end{aligned}  
 \end{equation} 
 \end{proposition}
 \begin{proof}
See the Appendix D.
 \end{proof}
When $Q_i$ adopts the \textit{Standard Proximal} form, the requirement for $\zeta_i$ in Theorem 1 and Corollary 1 is unique. But as for $\tau_i$, the conditions become different as summarized in Proposition 2. From (\ref{eq23}) we can conclude that $\tau_i$ is proportional to the number of agents $N$. For a larger decentralized network, the larger $\tau_i$ is required for all agents. Moreover, $\tau_i$ is also proportional to $d_i$, the degree of agent $i$. This is natural since the updated $\check{\bm{w}}_i$ of task $i$ with more connections affects the updating process of more agents. Hence to ensure the convergence, the step size should be reduced at agent $i$. 
 
\section{Mobility-Aware Content Preference Prediction}
In this section, we first introduce the mobility prediction model based on the tool of \textit{Markov renewal processes} \cite{WCNCY17,Lee2006ACM} in a decentralized setting. Then by adopting the adaptive learning model, the predicted mobility pattern is integrated in DRMTL model. For the sake of secrecy, sharing the mobility information of MTs directly among agents is not allowed.
\subsection{Mobility Prediction Model}
The wireless user-mobility prediction has been extensively studied in \cite{Akoush2007,scellato2011nextplace,gambs2012next}. Here we only focus on predicting the sojourn time for MTs. We consider the moving pattern for MTs $\mathcal{M}=\{1,...,M \}$ located within $\mathcal{V}$. According to \cite{WCNCY17}, we model MT $m$ mobility as $\{(S_{m,l},T_{m,l}):l\geq 0 \}$, to predict the moving path and the sojourn time within $\mathcal{V}$ for a future time window $[0,t_d]$, where $T_{m,l}$ is the time instant of the $l$-th transition ($T_{m,0}=0$), and $S_{m,l}\in\mathcal{V}$ is the state at the $l$-th transition. The initial state for MT $m$ is supposed to be $S_{m,0}$, and the time that MT $m$ has stayed at $S_{m,0}$ before time $0$ is $t_{m,0}$. In the decentralized setting, agent $i$ can only predict the movement of MTs $\mathcal{M}_i$ to its neighbor agents $j\in\mathcal{V}_i$, where $\mathcal{M}_i\subset \mathcal{M},\mathcal{M}_i\cap \mathcal{M}_j=\emptyset,i,j\in\mathcal{V},i\neq j$ and $\cup_{i\in\mathcal{V}}\mathcal{M}_i=\mathcal{M}$. Considering MT $m$ located at agent $i$ where $S_{m,0}=i$, the transition probability of the embedded Markov chain is denoted as $\mathbbm{P}^{m,i}\in\mathbbm{R}_{+}^{N\times N}$, which is a row stochastic matrix and $\mathbbm{P}^{m,i}_{i,j}=0$ if $j\notin\mathcal{V}_i$. 
%We adopt the average sojourn time to predict the moving pattern of MTs. 
Denote $\Psi_{m,i}(x)$ as the probability mass function (pmf) of the sojourn time for MT $m$ staying within agent $i$. 
%the average time then is evaluated by 
%\begin{equation}
%\phi_{m,i}=\sum\nolimits_{x=0}^{\infty}  x\Psi_{m,i}(x).
%\end{equation} 
%Thus, MT $m$ is expected to make a transition of Markov state, when the time duration staying within area $i$ exceeds $\phi_{m,i} $. 
We define the moving path set for MT $m$ at agent $i$ within time $[0,t_d]$ as $\mathcal{S}_m:=\{\bm{S}_{m,1},...,\bm{S}_{m,|\mathcal{S}_m|}\} (1\leq|\mathcal{S}_m|\leq d_i)$ with initial state $S_{m,j,0}=S_{m,0}$ for any path $j (1\leq j\leq |\mathcal{S}_m|  )$, where the $j$-th path is $\bm{S}_{m,j}:=[S_{m,0},...,S_{m,j,|\bm{S}_{m,j}|}]$.   
Since MT $m$ stays at area $i$ for time $T_{m,0}$, the $1$-st transition for all paths in $\mathcal{S}_m $ is predicted to occur at time instant
\begin{equation}
T_{m,1} =\sum\nolimits_{x=T_{m,0} }^{\infty} x\Psi _{m,i} (x) .
\end{equation} 
%$\mathcal{T}_{n,u,1} $ is the same among all possible paths and $\mathcal{T}_{n,u,0} =0$.
Specially, we adopt the average sojourn time to predict the transition.
Hence if $T_{m,1}\geq t_d$, we have $|\mathcal{S}_m|=1$ and $\bm{S}_{m,1}=i$ with path probability $\Pr(\bm{S}_{m,1})=1$, otherwise $|\mathcal{S}_m|=d_i$ and $\bm{S}_{m,j}:=[i,S_{m,j,1}]$ where $S_{m,j,1}\in\mathcal{V}_i$ with path probability
\begin{equation}
	\Pr\big( \bm{S}_{m,j} \big)=\mathbbm{P}^{m,i}_{i, S_{m,j,1}},~1\leq j\leq d_i.
\end{equation} 
% \begin{figure}[t]   
%	\label{[fig:2]}   
%	\begin{center}
%		\includegraphics[width=70 mm]{fig2.pdf}
%		\caption{.}
%	\end{center}
%\end{figure} 
%Without loss of generality, we assume that $t_d$ and $T_{m,1}$ has common factor $\frac{t_d}{K}$ where $K\in\mathbbm{N}_{+}$. Then we extend the path $\bm{S}_{m,j}$ to a new vector $\overline{ \bm{S}}_{m,j}\in \mathcal{V}^K$ with elements
%\begin{equation}
%\begin{aligned}
%& \overline{S}_{m,j,l}=S_{m,j,o}, ~ 1\leq l\leq K, o=\left\lceil\frac{1}{T_{m,1}}\left( \frac{lt_d}{K}-T_{m,1}  \right)\right\rceil_{+},
%\end{aligned}      
%\end{equation} 
%where $\lceil a\rceil_{+}$ returns $\lceil a\rceil$ if $a> 0$ otherwise $0$.
%It should be noted that $\overline{ \bm{S}}_{m,j}$ involves both \textit{Temporal} and \textit{Spatial} information, and its derivation is shown in Fig. 2. The probability satisfies $\Pr ( \overline{ \bm{S}}_{m,j} )=\Pr( \bm{S}_{m,i})$. 
Then for agent $i$ and MT $m\in\mathcal{M}_i$, we can derive the predicted residence time to agent $j $ as
\begin{equation}
r_{m,i\to j} = \big[ t_d-T_{m,1}\big]_{+}\mathbbm{P}_{i,j}^{m,i},~j\in\mathcal{V}_i,
%& \big( t_d-T_{m,1}\big)\mathbbm{P}_{i,j}^{m,i},~\text{otherwise}.
\end{equation}
%where $\bm{1}_j = [j,...,j]_{1\times K}$, and $\bm{x}\oplus\bm{y} $ is the exclusive-or
%(EOR) operation of vector $\bm{x}$ and $\bm{y}$. The elements of $\bm{x}$ and $\bm{y}$ need not necessarily to be 0 or 1, and $(\bm{x}\oplus\bm{y})(\bm{x}\oplus\bm{y})^T$ counts the same elements from $\bm{x}$ and $ \bm{y}$. Hence 
%The $r_{m,i\to j}$ represents the predicted residence time for MT $m$ at agent $j$ during time window $[0,t_d]$.
%%For simplicity, we denote $r_{m,j\to i}$ as $r_{i,j}$. 
where $[a]_{+}=a$ if $a>0$ otherwise $0$. The prediction accuracy of mobility patterns for our provided model will be analyzed in Section V. 

\subsection{Mobility-Aware Preference Prediction}
 
%{\color{blue}{Integrate the mobility pattern into the learning weights $c_{i,j} $ and $\eta_i$ design. Explain why we need to control the similarity of weights with integrating mobility pattern of areas. Explain the impact of mobility on learning content preference in future time duration $[0,T]$.}}

Inspired by the adaptive learning method \cite{Ren2018ICML} and distributed MTL model in \cite{Li2016TAES}, we modify the DRMTL in (P1) to the following problem (P2), which integrates the predicted mobility pattern of MTs,
 		
 \begin{equation}
 \begin{aligned}
  \min_{\check{\bm{W}} ,\hat{\bm{W}} }  &~\sum\nolimits_{i\in\mathcal{V} }\bigg(\tilde{f}_i\big(\check{\bm{w}}_i,\hat{\bm{w}}_i \big) + \frac{\mu_1}{2}\big\|\check{\bm{w}}_i \big\|^2+\frac{\mu_2}{2}\big\|\hat{\bm{w}}_i \big\|^2+ \\&\vphantom{\frac{1}{2}}   \qquad\qquad  \quad ~~~~ \frac{\mu_3}{2}\sum\nolimits_{j\in\mathcal{V}_i }c_{j,i}\big\|\bm{w}_{j,i}^{\text{loc}}- \big(\check{\bm{w}}_i+\hat{\bm{w}}_i \big)\big\|^2  \bigg),\\
 s.t. &~ \bm{A}  \check{\bm{W}} =\bm{0},\\
 &~\bm{w}_{i,j}^{\text{loc}}=\arg\min_{\bm{w}} \tilde{f}_{i\to j} \big(\bm{w}\big)+\frac{\mu_{12}}{2}\big\|\bm{w}\big\|^2,~i\in\mathcal{V}, j\in\mathcal{V}_i,
 \end{aligned}    \tag{P2}  
 \end{equation} 
 where 
 \begin{align}
  \tilde{f}_i\big(\check{\bm{w}}_i,\hat{\bm{w}}_i\big) =& \frac{1}{b_i}\sum\nolimits_{l=1}^{b_i}\phi_{i,l}\ell\big(\check{\bm{w}}_i,\hat{\bm{w}}_i,x_{i,l},y_{i,l}\big),\label{eq27} \\
 \tilde{f}_{i\to j } \big(\bm{w}\big)=&\frac{1}{b_i}\sum\nolimits_{l=1}^{b_i}\phi_{i\to j,l} \ell\big(\bm{w},x_{i,l},y_{i,l}\big).\label{eq28}
 \end{align}
 In (P2), $\bm{w}_{i, j}^{\text{loc}}$ denotes the weights transfered from agent $i$ to $j$. It contains the information for the leaving crowds of $\mathcal{M}_i$ to $j$. 
 The adaptive parameters $\phi_{i,l},\phi_{i\to j,l}\in( 0,1]$ are chosen by agent $i$ according to the predicted moving pattern of MTs $\mathcal{M}_i$. Hence (\ref{eq27}) and (\ref{eq28}) can be seen as a reweighting of the examples as \cite{Ren2018ICML}.  
 $\{c_{i,j}\}$ are the non-negative task combiners, which control the similarity of transformed weights and local ones of neighboring agents \cite{Li2016TAES}. $c_{i,j}$ is generated at agent $i$ according to the leaving crowd to agent $j$. To present the settings of $\phi_{i,l},\phi_{i\to j,l}$, we assume that sample $(x_{i,l},y_{i,l})$ is associated with MT $m(\in\mathcal{M}_i)$. Then we design 
 \begin{equation}
 	\phi_{i\to j,l} =\frac{r_{m,i\to j}}{t_d},~j\in\mathcal{V}_i; ~
 	\phi_{i,l}=1-\frac{1}{t_d}\sum\nolimits_{j\in\mathcal{V}_i}r_{m,i\to j}.\label{eq29}
 \end{equation}
 It is easy to verify that $\phi_{i,l}+\sum_{j\in\mathcal{V}_i} \phi_{i\to j,l}=1$. We define an $N\times N$ matrix $\bm{c}$ with entries $c_{i,j}$, which satisfies 
 $\sum\nolimits_{j}c_{i,j}\leq 1, ~c_{i,j}=0 ~\text{if} ~j\notin \mathcal{V}_i,~\forall i\in\mathcal{V}. $  The matrix $\bm{c}$ does not need to be symmetrical. We set the intertask combiners by
 \begin{equation}
 c_{i,j}=\left\{   
 \begin{aligned}
 &\frac{1}{d_i} \bigg[ 1-\exp\bigg(-\upsilon \sum\nolimits_{m\in\mathcal{M}_i}r_{m,i\to j}\bigg) \bigg] ,~j\in\mathcal{V}_i;\\
 &0,~j\notin \mathcal{V}_i,
%  &0,~\text{otherwise},
 \end{aligned}
 \right.\label{eq30}
 \end{equation}
 where $\upsilon$ is a positive constant.

% \begin{algorithm} [t]
%	\SetAlgoLined
%	\KwIn{ $\{\mathcal{D}_i,\tau_i,\zeta_i|i\in\mathcal{V}\},\rho,\gamma,\mu_1,\mu_2$,} 
%	%	\KwIn{Data $\{D_t|t=1,\cdots m\}$}
%	\KwOut{$\{\check{\bm{w}}_i,\hat{\bm{w}}_i|i\in\mathcal{V}\}$} 
%	Initialization $ \{\check{\bm{w}}^0_i,\hat{\bm{w}}_i^0 |i\in\mathcal{V} \} $ randomly and $\lambda^0=\bm{0}$\; 	
%	{\bfseries Agents $i=1$ to $N$}:\\
%	Calculate $\{\bm{w}^{\text{loc}}_{ij}  |j\in\mathcal{N}_i \}$ according to (\ref{ }).\\
%	\For{$k=0,1,...$   }{ 
%		{\bfseries Agents $i=1$ to $N$}:\\
%		Update $\check{\bm{w}}^{k+1}_i$ \textit{in parallel} according to (\ref{})\; 		 	
%		Communicate $\check{\bm{w}}_i^{k+1}$ with neighboring agents and update $\lambda^{k+1}$ according to (\ref{ }).\\		  
%		{\bfseries Agents $i=1$ to $N$}:\\   
%		Update $\hat{\bm{w}}^{k+1}_i$ \textit{in parallel} according to (\ref{ })\; 		
%	} 	
%	\caption{}
%\end{algorithm}

In what follows, we will state the rationality of settings (\ref{eq29}) and (\ref{eq30}). From (\ref{eq27}), $\phi_{i,l}$ controls the importance of sample $(x_{i,l},y_{i,l})$ in parameter training at agent $i$. While $\phi_{i\to j,l}$ controls the importance of $(x_{i,l},y_{i,l})$ on the transfered model parameter to agent $j$, which indirectly influence the output model parameters at agent $j$.
Considering MT $m\in\mathcal{M}_i$ related to $(x_{i,l},y_{i,l})$, the sample will be more important for training the model at agent $i$ if $m$ is predicted to stay longer time during $[0,t_d]$. Hence a larger $\phi_{i,l}$ is resulted. This also explains the choice of $\phi_{i\to j,l}$. On the other hand, the task combiner $c_{i,j}$ decides the amount of information provided by agent $i$ to $j$. If the flow of MTs from agent $i$ to $j$ is predicted to be dense, we set a large intertask combiner $c_{i,j}$ to enforce the model parameter of $j$ similar to the transfered $\bm{w}_{i,j}^{\text{loc}}$.

% {{\color{red}{state the rationality of parameters}} \color{blue}{approximately choose $C$,clarify the reason...$c_{i,j}$ controls the similarity between $\hat{\bm{w}}_i $ and $\hat{\bm{w}}_j $. $c_{i,j}$ is determined at agent $j$ which reflects the predicted moving flow to area $i$. ... $\phi_i$}} 
 
% {\color{blue}integrate mobility prediction into the selection of combiners $c_{i,j}$. Homogeneous request frequency for all mobile terminals (MTs)... Under this assumption, the request times is only related to residence time of MTs. }

% Considering all outflows of agent $j$, we determine the adaptive parameter as
% \begin{equation}
% \phi_j =  \exp\big( -\nu \sum\nolimits_{i\in\mathcal{V}_j}r_{i,j}  \big) .
% \end{equation}
% The $\nu$ is also a positive constant predetermined by agents.
 
% {\color{blue}{Integrate the mobility pattern into the learning weights $c_{i,j} $ and $\phi_i$ design. Explain why we need to control the similarity of weights with integrating mobility pattern of areas. Explain the impact of mobility on learning content preference in future time duration $[0,T]$. }}
 \begin{proposition}
 	With intertask combiners and adaptive parameter designed at (\ref{eq29}) and (\ref{eq30}), (P2) and (P1) are consistent when there is no MT flow across areas. 
 \end{proposition}
 \begin{proof}
 	When there does not exist the transition of MTs among agents, $r_{m,i\to j}=0,\forall i\in\mathcal{V},j\in\mathcal{V}_i $ and hence $\phi_{i,l} = 1,\phi_{i\to j,l}= c_{i,j}=0,\forall i\in\mathcal{V}, j\in\mathcal{V}_j,l=1,...,b_i$. Then $\tilde{f}_i(\check{\bm{w}}_i,\hat{\bm{w}}_i)=f_i(\check{\bm{w}}_i,\hat{\bm{w}}_i)$ and (P2) reduces to (P1).
 \end{proof}

 \begin{algorithm}[h]
 	\caption{} 
 	\begin{algorithmic}[1]
 		\STATE \textbf{initialize}: set $ \{\check{\bm{w}}^0_i,\hat{\bm{w}}_i^0 |i\in\mathcal{V} \} $ randomly and $\lambda^0=\bm{0}$.  
 		\STATE {\bfseries agents $i=1$ to $N$}:
 		\STATE calculate $\{\phi_{i,l},\phi_{i\to j,l},c_{i,j},\bm{w}^{\text{loc}}_{i,j} |l=1,...,b_i,j\in\mathcal{V}_i \}$ and communicate $\bm{w}^{\text{loc}}_{i,j} $ to agent $j\in\mathcal{V}_i$.
 		\FOR{$k=0,1,...$ } 
 		% 		\STATE {\bfseries agents $i=1$ to $N$}:
 		% 		\STATE update $\check{\bm{w}}^{k+1}_i$ with (\ref{eq32}) \textit{in parallel} ; 
 		% 		\STATE communicate $\check{\bm{w}}_i^{k+1}$ with neighboring agents and update $\lambda^{k+1}$ with (\ref{eq33}).	 
 		% 		\STATE {\bfseries agents $i=1$ to $N$}:
 		% 		\STATE update $\hat{\bm{w}}^{k+1}_i$ with (\ref{eq34}) \textit{in parallel}; 
 		\STATE follow steps 3-7 in Algorithm 1 but substituting $\mathcal{L}_{\rho}^1$ with $\mathcal{L}_{\rho}^2$.
 		\ENDFOR 
 	\end{algorithmic} 
 \end{algorithm}

  The augmented Lagrange function for (P2) is given by
 \begin{equation}
 \begin{aligned}
 \mathcal{L}_{\rho}^2\big(\check{\bm{W}},\hat{\bm{W}},\lambda\big)=\sum\nolimits_{i }\bigg(  \tilde{f}_i\big(\check{\bm{w}}_i,\hat{\bm{w}}_i \big) + \frac{\mu_1}{2}\big\|\check{\bm{w}}_i \big\|^2+\frac{\mu_2}{2}\big\|\hat{\bm{w}}_i \big\|^2+~~~~~~ &\\\frac{\mu_3}{2}\sum\nolimits_{j\in\mathcal{V}_i }c_{j,i}\big\|\bm{w}_{j,i}^{\text{loc}}-\big(\check{\bm{w}}_i+ \hat{\bm{w}}_i\big) \big\|^2+\lambda^T\bm{A}\check{\bm{W}} + \frac{\rho}{2}\big\| \bm{A}\check{\bm{W}} \big\|^2\bigg).&
 \end{aligned}
 \end{equation}
Adopting hybrid Jacobi and Gauss-Seidel type ADMM as Algorithm 1, the update for (P2) follows (\ref{eq3})-(\ref{eq5}) but substituting $\mathcal{L}_{\rho}^1$ with $\mathcal{L}^2_{\rho}$. 
%\begin{align}
%\check{\bm{w}}_i^{k+1} :=&\arg\min \mathcal{L}_{\rho}^{\dagger}  \big(\check{\bm{w}}_i,\hat{\bm{w}}_i^k,\check{\bm{W}}_{-i}^k,\lambda^k  \big)+\frac{1}{2} \big\| \check{\bm{w}}_i-\check{\bm{w}}^k_i  \big\|^2_{P_i}, \label{eq32}\\
%\vphantom{\frac{1}{2}} \lambda^{k+1} :=& \lambda^{k} + \gamma \rho \bm{A}\check{\bm{W}}^{k+1},\label{eq33}\\
%\hat{\bm{w}}_i^{k+1} :=&\arg\min \mathcal{L}^{\dagger}_{\rho} \big(\hat{\bm{w}}_i,\check{\bm{w}}_i^{k+1}  \big)+\frac{1}{2} \big\| \hat{\bm{w}}_i-\hat{\bm{w}}^k_i  \big\|^2_{Q_i}.\label{eq34}
%\end{align} 
% where $\lambda$ is a Lagrange multiplier and $\rho(>0)$ is the step penalty.  
% To be more compactly, the objective of P2 can be expressed as the variation of (1) with an additive quadratic term
% \begin{equation}
% 	\sum\nolimits_{i}\big( \tilde{f}_i\big(\check{\bm{w}}_i,\hat{\bm{w}}_i  \big) +g\big(\check{\bm{w}}_i\big)+h\big(\hat{\bm{w}}_i\big)\big)+\hat{\bm{W}}^T\Theta\hat{\bm{W}},
% \end{equation}
% where $\tilde{f}_i(\check{\bm{w}}_i,\hat{\bm{w}}_i)=\phi_i f_i(\check{\bm{w}}_i,\hat{\bm{w}}_i)$, and the entries of $\Theta$ satisfy 
% $$
%    \theta_{ii}=\frac{\mu_3}{2}\sum\nolimits_{j}(c_{ij}+c_{ji});~\theta_{ij} = -\mu_3c_{ij},~i\neq j.  
% $$  
Then we provide Algorithm 2 to obtain the optimal solution for (P2). Compared with the process of Algorithm 1, the difference is that agent $i(\in\mathcal{V})$ needs to calculate the transfered weights $\bm{w}_{i,j}^{\text{loc}}$ first. 

%\subsection{Selection of Adaptive Parameters and Combiners}

\subsection{Convergence Analysis}
To analyze the convergence of Algorithm 2, we define
\begin{equation}
	\mathbbm{f}_i\big(\check{\bm{w}}_i,\hat{\bm{w}}_i\big)=\tilde{f}_i\big(\check{\bm{w}}_i,\hat{\bm{w}}_i\big) + \frac{\mu_3}{2}\sum\nolimits_{j\in\mathcal{V}_i}c_{j,i}\big\|\bm{w}_{j,i}^{\text{loc}}-\big(\check{\bm{w}}_i+\hat{\bm{w}}_i \big) \big\|^2,\label{eq35}
\end{equation}
and $\mathbb{F}(\check{\bm{W}},\hat{\bm{W}})=\sum\nolimits_{i}(\mathbbm{f}_i(\check{\bm{w}}_i,\hat{\bm{w}}_i)+ g(\check{\bm{w}}_i)+h(\hat{\bm{w}}_i)  )$.
%, it is easy to verify that $\tilde{f}_i(\check{\bm{w}}_i,\hat{\bm{w}}_i)$ satisfies Assumption 2 and Lemma 1 but with the Lipschitz constant $\phi_iL_i$.
Then we follow the same argumentation as in the convergence proof for Algorithm 1.
\begin{assumption}
	$\tilde{f}_i$ and $\tilde{f}_{i\to j}$ are differentiable and jointly convex over $\check{\bm{w}}_i$ and $\hat{\bm{w}}_i$. The gradient $\nabla \tilde{f}_i$ is Lipschitz continuous with constant $\tilde{C}_i$.
	
\end{assumption}
With Assumption 3, we can present the following property of $\mathbbm{f}_i$ by Lemma 3.
\begin{lemma}
	 $\mathbbm{f}_i $ satisfies Assumption 2 and Lemma 1 but with Lipschitz constant
	  \begin{equation}
	  	\mathbbm{C}_i = \sqrt{2\tilde{C}_i^2 + 4\mu_3^2\left(\sum\nolimits_{j\in\mathcal{V}_i }c_{j,i} \right)^2}.
	  \end{equation}	 
\end{lemma}
\begin{proof}
	See the Appendix E.
\end{proof}

 \begin{theorem} (Convergence of Algorithm 2)
% 	 Assuming $m> \mu_3\sum\nolimits_{j }(c_{ij}-c_{ji})$ and  	
 	Following the conditions in Corollary 1, letting $\tau_i$ satisfy (\ref{eq23}) and
 	\begin{equation}
 	 \zeta_i>\frac{\mathbbm{C}_i }{m}\big(\mathbbm{C}_i +m\big),
 	\end{equation}  
% 	where
% 	\begin{equation}
% 		\left\{ \begin{aligned}
% 		\varsigma_i =& \frac{m-\mu_3\sum\nolimits_{j }(c_{ij}-c_{ji})}{1 + \mu_3\sum\nolimits_{j }c_{ij}},\\
% 		\phi_i =& \frac{\phi_iL_i}{2\varsigma_i}\big( \phi_iL_i + \varsigma_i \big) + \mu_3\sum\nolimits_{j }\frac{c_{ji}}{2\varsigma_j},
% 		\end{aligned}\right. 
% 	\end{equation}
 	then the sequence $\bm{u}^k$ generated by Algorithm 2 converges to the global optimal solution $\bm{u}^*$ of problem (P2) with rate (\ref{eq22}).
% 	\begin{equation}
% 	\begin{aligned}
% 	 \mathbb{F} \big(\overline{\bm{Z}}^k \big)- \mathbb{F}\big(\bm{Z} ^* \big)  \leq  \frac{1}{2k}\bigg(  \big\|  \check{\bm{W}}^0-\check{\bm{W}}^* \big\|^2_{\bm{G}_1^{\dagger}} +  \big\|  \hat{\bm{W}}^0-\hat{\bm{W}}^* \big\|^2_{\bm{G}_2 } \bigg),
% 	\end{aligned}
% 	\end{equation}
% 	where $\overline{\bm{Z}} ^k=\frac{1}{k}\sum_{i=1}^{k}\bm{Z} ^i$.
 \end{theorem}  
 \begin{proof}
With Lemma 3, it can be proved similarly as Theorem 1 and Corollary 1 but substituting $C_i$ with $\mathbbm{C}_i  $.
 \end{proof}
%The condition (\ref{label}) is consistent with (\ref{label}), since by letting $c_{ij}=0$ two results coincide.   
%Similarly, by replacing $\bm{G}_3$ with $ \Psi=2\text{blkdiag}( \phi_1I,...,\phi_NI)$, we can deduce the convergence rate as Corollary 1 and find the corresponding requirements of parameters. But due to the ...
It is worth noting that the conditions derived in Theorems 1 and 2, Corollary 1 and Propositions 1 and 2 are all sufficient to make Algorithms 1 and 2 converge. Moreover, the proofs can be extended to the matrix case $\check{\bm{w}}_i,\hat{\bm{w}}_i\in\mathbbm{R}^{n_1\times n_2 }$ by substituting the product $a^Tb$ by $\langle a,b  \rangle =\text{tr}(a^Tb)$.

\section{Simulation and Discussion}
In this section, we will provide simulations based on real dataset for the presented algorithms, and give discussions on the numerical results.
%The convergence of proposed Algorithms 1 and 2 is first verified. Then we validate the feasibility of Algorithm 2 in predicting content preference for distributed agents with moving MTs. Finally, we discuss the prediction accuracy on mobility pattern of MTs based on real traces, and carry out the hit ratio comparison.

% {\color{blue}The mobility pattern of MTs based on real traces is first discussed. Then, through the regression on content preferences for two MT groups, we validate the performance improvement for the proposed ML-ELM. Finally, the hit-rate comparison between MPPC, Most popular caching (MPC) and random caching (RC) \cite{CMG16} with perfect and predicted content preference are given respectively.  }

\subsection{Convergence Experiment}
We first evaluate the convergence of Algorithms 1 and 2 by using the least square loss at agent $i$ as
\begin{equation}
	\ell\big(\check{\bm{w}}_i,\hat{\bm{w}}_i,x_{i,l},y_{i,l} \big)= \frac{1}{2}\big\| \big(\check{\bm{w}}_i + \hat{\bm{w}}_i \big)^Tx_{i,l} - y_{i,l} \big\| ^2.\label{eq38}
\end{equation}
With (\ref{eq38}) it is easy to verify that Assumptions 2 and 3 can be satisfied. Denoting $X_i = [x_{i,1},...,x_{i,b_i}]$ and $Y_{i}=[y_{i,1},...,y_{i,b_i}]$, the local objective becomes
$f_i(\check{\bm{w}}_i,\hat{\bm{w}}_i) = \frac{1}{2b_i}\|(\check{\bm{w}}_i+\hat{\bm{w}}_i)^TX_i-Y_i  \|^2$.
\begin{figure} [h]
	\vskip 0.2in
	\begin{center}
		\centerline{\includegraphics[width=85 mm]{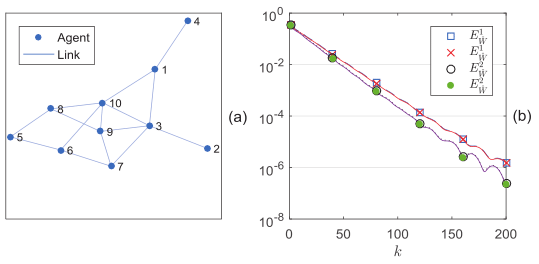}}
		\caption{ (a) the connected graph $\mathcal{G}$; (b) learning accuracy of Algorithms 1 and 2.}
		\label{fig2}
	\end{center}
	\vskip -0.2in
\end{figure} 

In the simulation setup, we let the number of agents $N=10$ with connections $|\mathcal{E}|=15$, the number of training samples $b_i=10~(\forall i\in\mathcal{V})$, input and output dimension $n=10$, $\nu=1$. The regularization parameters are $\mu_1=\mu_2=\mu_3=\mu_{12}=1$ while step penalties are $\tau_i=\zeta_i=1$ and $\rho=\gamma=1$. The connected graph $\mathcal{G}$ is generated randomly shown as Fig. \ref{fig2} (a). The entries of input $x_{i,l}$ and output $y_{i,l}$ are generated randomly according to uniform distribution $\mathcal{U}(0,1)$. Besides, we generate $\phi_{i\to j,l}\sim\mathcal{U}(0,\frac{1}{N})$ and hence $\phi_{i,l} = 1-\sum_{j\in\mathcal{V}_i}\phi_{i \to j,l}$
% the diagonals of $\Phi_{i\to j}$ according to $\mathcal{U}(0,\frac{1}{N}) $ and hence $\Phi_i = I-\sum_{j\in\mathcal{V}_i}\Phi_{i \to j}$. 
The intertask combiners for agent $i$ are generated by $c_{i,j}\sim\mathcal{U}(0,\frac{1}{N})(j\in\mathcal{V}_i)$.

% \begin{figure*} 
%	% 	\vskip 0.2in
%	\begin{center}
%		\includegraphics[width=170 mm]{MT_MTL.pdf} 
%		\caption{(a) locations of MTs with $K=2$; (b) content preference $P_1$ and $P_2$ for the two groups; (c) the predicted content preference $\hat{p}_{1,2} $ by Algorithms 1 when MTs stay; (d) the common preference $\tilde{p}_{1,2}$ predicted by Algorithm 1 when MTs stay; (e) predicted content preference $\hat{p}_1$ and $\hat{p}_1^{\dagger}$ by Algorithms 1 and 2, and observed preference $p_1$ at agent $1$; (f) predicted content preference $\hat{p}_2$ and $\hat{p}_2^{\dagger}$ by Algorithms 1 and 2, and observed preference $p_2$ at agent $2$.}
%		\label{fig3}
%	\end{center}
%	\vskip -0.2in
%\end{figure*} 

Denote the optimal solutions for (P1) and (P2) as $\bm{u}^1_{\text{opt}}$ and $\bm{u}^2_{\text{opt}}$ respectively, which are obtained by solving centralized cases for (P1) and (P2) with adopting Alternating Optimization (AO) method \cite{Beck2015OnTC}. 
Then we define the accuracy with respect to $\check{\bm{W}}$ and $\hat{\bm{W}}$ in Algorithms 1 and 2 as
\begin{equation}
	E_{\check{\bm{W}}}^{1,2}(k) = \sqrt{ \frac{1}{N n} \big\|\check{\bm{W}}^k-\check{\bm{W}}_{\text{opt}}^{1,2} \big\|^2},~ E_{\hat{\bm{W}}}^{1,2}(k) = \sqrt{ \frac{1}{N n} \big\|\hat{\bm{W}}^k-\hat{\bm{W}}_{\text{opt}}^{1,2} \big\|^2},
\end{equation}
which are shown in Fig. \ref{fig2} (b). As we can observe, all the accuracy $E_{\check{\bm{W}}}^{1,2}$ and $E_{\hat{\bm{W}}}^{1,2}$ reduce with iteration $k$, which supports the convergence analysis in Theorems 1 and 2.

\subsection{Experiments on Real Dataset}

The mobility prediction model based on the real trace data set \textit{Geolife Trajectories} \cite{Zheng2009ACM,Zheng2008ACM,Zheng2010} is investigated. We select $M=20$ MTs out of the dataset with total $14587$ paths. These paths were recorded during $9$ months. As shown in Fig. \ref{Fig4} (a), the locations of MTs are exemplified based on GPS records (latitude and longitude), and some of the paths of each MT are shown in different colored dots respectively.  
%There are $137$ paths recorded in the dataset within the windowed area, and each path is shown in different colored dots respectively. 
The windowed area is roughly $5 \times 5\text{km}^2$ large between latitude $(39.97, 40.02)$
and longitude $(116.30, 116.35)$. Assume that the agent $i$ covers an square area with size $\frac{5}{s}\times \frac{5}{s}\text{km}^2$. Then in considered region shown in Fig. \ref{Fig4} (a), we have agent $N=s^2$. Since traces of MT are repeated, the pmf $\Psi _{m,i} (x)(i\in\mathcal{V})$ can be calculated statistically. 
 \begin{figure}[h] 
	\vskip 0.2in
	\begin{center}
		\centerline{\includegraphics[width=97 mm]{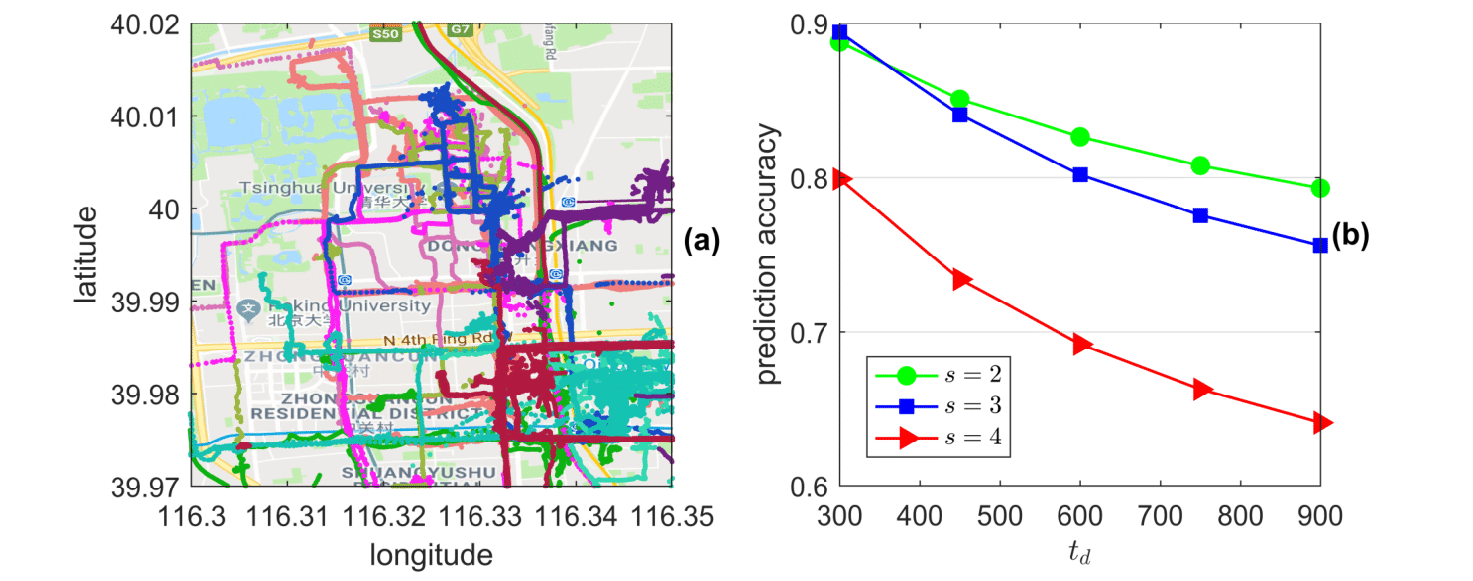}}
		\caption{(a) real traces of MTs; (b) accuracy of distributed mobility prediction based on \textit{Markov renew process} model.}
		\label{Fig4}
	\end{center}
	\vskip -0.2in
\end{figure} 
 
In Fig. \ref{Fig4} (b), we verify the prediction accuracy of the \textit{Markov renewal process} based mobility prediction model for these $ 20$ MTs. We average the accuracy over MTs by conducting the prediction $100$ times for the time window $[0,t_d]$. In prediction of MT $m$, we generate $T_{m,0}\sim\Psi_{m,S_{m,0}}$. 
We evaluate the prediction accuracy by counting the difference between real path and all the possible predicted paths, the detail of which can be found in \cite{WCNCY17}. 
As shown in Fig. \ref{Fig4} (b), the average accuracy for considered MTs degrades when we predict for a larger time window $t_d$, which is due to the increased uncertainty of MT mobility. Moreover, when the region is covered by more agents, e. g. with a smaller $s$, the prediction accuracy reduces. This is because increased agents may expand the number of possible paths within $t_d$. Thus it is harder to predict the movements.  
As a conclusion, for the considered $t_d$ and $s$, the accuracy can be guaranteed at least $0.63$ with the presented mobility model.
 
%We assume the caching capacity for agents is $\theta$. In proactive caching, each agent stores the contents at time $0$ according to the following tow policies:
%\begin{enumerate}[leftmargin=*]
% \item {\textit{Most popular caching}} (MPC) \cite{bibid}: each agent proactively stores content according to descending order of content preference until $\theta$ is occupied;
% \item{\textit{Random caching}} (RC) \cite{bibid}: each agent proactively stores content in random manner until $\theta$ is occupied.
%\end{enumerate} 
%For the MPC method, we consider two approaches, the Algorithms 1 and 2, to predict the content preference for agents. 
 
%The first one is that each agent learn preference separately by solving 
%\begin{equation}
%\begin{aligned}
%\bm{w}_i =& \arg\min_{\bm{w}}~f_i\big(\bm{w}\big)+\frac{\mu_4}{2}\big\|\bm{w} \big\|^2\\=&\arg\min_{\bm{w}}~\frac{1}{ 2b_i}\big\|\bm{w}^TX_i-Y_i\big\|^2+\frac{\mu_4}{2}\big\|\bm{w} \big\|^2,~i\in\mathcal{V}.
%\end{aligned}	\label{eq48}
%\end{equation}
%The optimal solution of (\ref{eq48}) is given by $\bm{w}_i^{*}=\frac{1}{ b_i}(\frac{1}{  b_i}X_i^TX_i+\mu_4I)^{-1}X_i^TY_i$.
%We denote the predicted preference by this method as $\hat{p}_i^{s}$. 
%While another two methods are from proposed Algorithms 1 and 2. 
%Since the stored contents are randomly selected for RC policy, there is no need to conduct the prediction of preference..

  \begin{figure} [h]
	\vskip 0.2in
	\begin{center}
		\centerline{\includegraphics[width=88 mm]{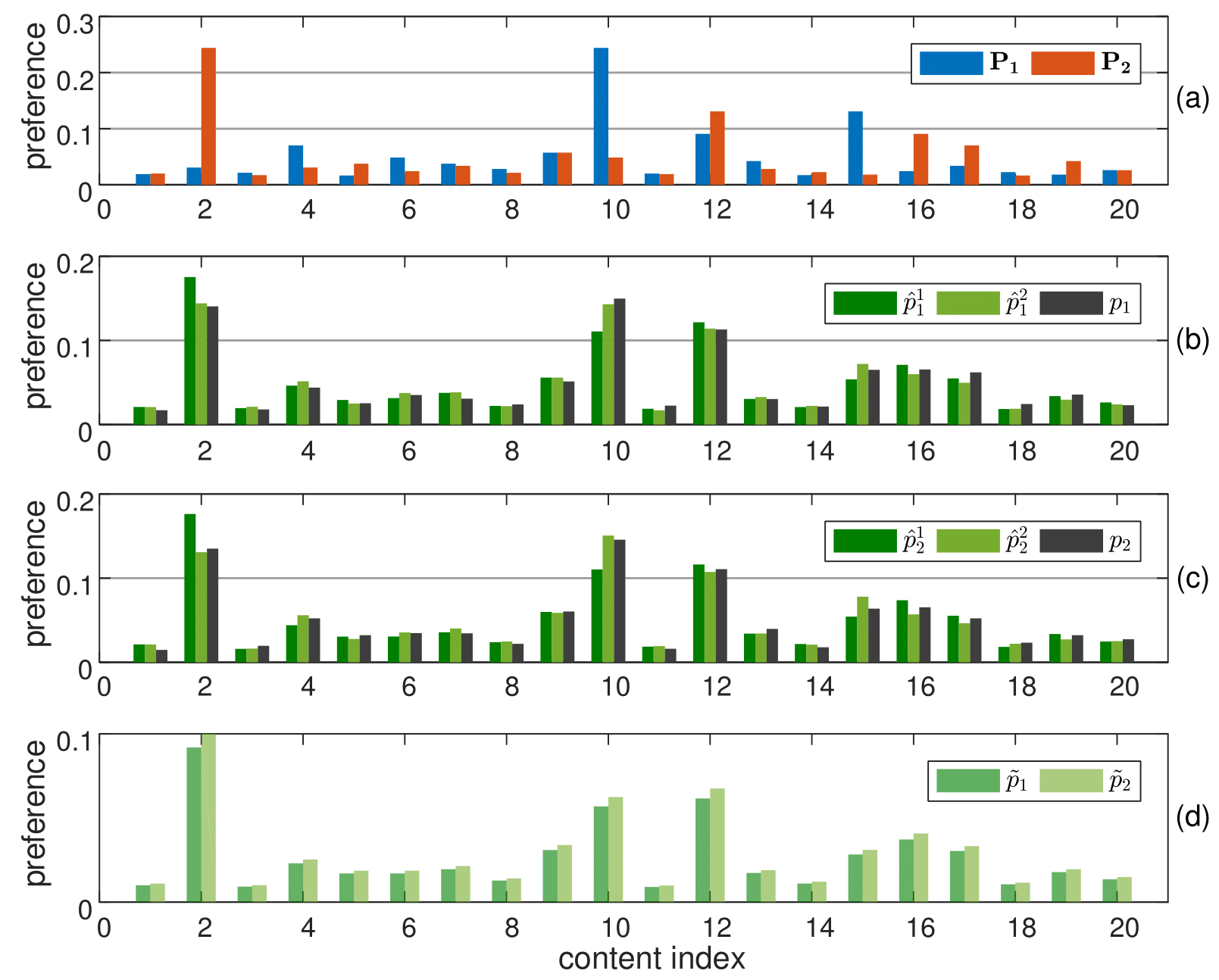}}
		\caption{(a) content preference for $K=2$ groups; (b), (c) predicted content preference $\hat{p}^{1,2}_{1,2}$ by Algorithms 1 and 2, and observed preference $p_{1,2}$ at agents 1 and 2; (d) common preference $\tilde{p}_{1,2}$ predicted at agents 1 and 2.}
		\label{fig5}
	\end{center}
	\vskip -0.2in
\end{figure}

Then we evaluate the prediction results of Algorithms 1 and 2 with predicted moving patterns of MTs. We let $F=20$, $s=3$, $K=2$, $\iota_i=0.9$, $t_d=30$min and randomly choose one path for each MT. Besides we generate $T_{m,0}\sim\Psi_{m,S_{m,0}}$. Each MT in $\mathcal{M}$ is randomly assigned to one of the group in $\mathcal{K}$. The content preference for group $i\in\mathcal{K}$ follows a Zipf-like distribution with shape parameter $\iota_i$,
\begin{equation}
	 \bm{\text{P}}_{i }\big(\pi_{i,f}\big) = \frac{f^{-\iota_i}}{\sum\nolimits_{l\in\mathcal{F}}l^{-\iota_i}},~i\in\mathcal{K},~f\in\mathcal{F},
\end{equation}
where $\bm{\pi}_i$ is a random permutation of $\mathcal{F}$ \cite{Guo2017TC}. The content preference for two groups are presented in Fig. \ref{fig5} (a). The dataset $\mathcal{D}_i$ at agent $i$ is collected in time $[-t_d,0]$. We assume the request frequency for MTs is homogeneous as $R_f = 2$/min. Considering the sample $(x_{i,l},y_{i,l})$, which is a record for one request at agent $i$, the $x_{i,l}$ represents the feature of associated MT while $y_{i,l}$ records the requested file. Without loss of generality, we assume the MTs with similar features have analogous preference. Hence we use the index of groups of MT to generate $x_{i,l}\in\{0,1\}^{ K\times 1}$, e.g., the associated MT belongs to group 1, then $x_{i,l}=[1, 0]^T$ otherwise $x_{i,l}=[0,1]^T$. We denote the target as $y_{i,l}\in\{0,1\}^{ F\times 1}$, e.g., if $f$ is the requested file, then $y_{i,l,f}=1$ while the other elements are $0$. We denote $p_i$ as the content preference observed by agent $i(\in\mathcal{V})$ during time $[0,t_d]$. We run Algorithms 1 and 2 for $300$ iterations with parameter $\mu_1=\mu_2=0.1$, $\mu_{12}=0.01$, while the other parameters are the same with previous subsection.
The predicted content preference by agent $i$ is calculated by $\hat{p}_i^{1,2}=\frac{\sum_{m\in\mathcal{M}_i} (\check{\bm{w}}_i^{1,2}+\hat{\bm{w}}_i^{1,2})^Tx_m}{ \sum_{m\in\mathcal{M}_i} \bm{1}\cdot(\check{\bm{w}}^{1,2}_i+\hat{\bm{w}}^{1,2}_i)^Tx_m  }$, where $\check{\bm{w}}_i^{1,2}$ and $\hat{\bm{w}}_i^{1,2}$ are obtained by Algorithms 1 and 2 respectively and $\bm{1}=[1,...,1]$. $x_m(m\in\mathcal{M}_i)$ is the inputs observed at time $0$ for agent $i$.
 In particular, we investigate the predicted preference for agents 1 and 2. $ p_{1,2}$ and $\hat{p}_{1,2}^{1,2}$ are presented in Fig. \ref{fig5} (b) and (c), which illustrate that without considering MTs mobility, both $\hat{p}_1^1$ and $\hat{p}_2^1$ make wrong prediction on the most popular content. But $\hat{p}^2_1$ and $\hat{p}_2^2$ can give accurate prediction. The common preference predicted at agent $i$ with Algorithm 2 is given by $\tilde{p}_i = \frac{\sum_{m\in\mathcal{M}_i} (\check{\bm{w}}^2_i)^Tx_m}{ \sum_{m\in\mathcal{M}_i}\bm{1} \cdot (\check{\bm{w}}^2_i+\hat{\bm{w}}^2_i)^Tx_m }$. $\tilde{p}_1$ and $\tilde{p}_2$ are shown in Fig. \ref{fig5} (d). It demonstrates that the basis $\bm{w}_0$ learned jointly across agents can capture the common content interests.

%The predicted preference $\hat{p}_i$ and $\hat{p}_i^{\dagger}$ for agents 1 and 2, together with the observed preference $p_i$, are presented in Fig. \ref{fig5} (b) and (c) respectively. 
 
Moreover, we average the prediction error on learned preference over the agents by carrying out the simulation for $100$ times. 
%and generate the sojourn time $T_{m,0}\sim\Psi_{m,S_{m,0}}$. Then Same as the previous sub-section, the training samples from each MT are generated according to $\text{P}_i$ with number $T_{m,0}R_f$. 
The average estimation error of Algorithms 1 and 2 is defined by
\begin{equation}\label{eq48}
	\varepsilon^{1,2} = \frac{1}{N}\sum\nolimits_{i\in\mathcal{V}}\big\| \hat{p}_i^{1,2}-p_i \big\|_1.
\end{equation}
As presented in Fig. \ref{fig6} (a) and (b), 
%when $s=\{2,3\}$, $M=20$, $K=2$, $\iota_i=\{0.9,1.2\}$ and $F$ in range $[10,50]$, 
the error of predicted geography preference achieved by the mobility-aware DRMTL model is always smaller than that of DRMTL. With growing $F$, the error $\varepsilon^1$ and $\varepsilon^2$ both exhibit an increasing trend. This is because more contents need to be predicted for a larger $F$ but the number of training samples is fixed. Moreover, due to the fact that it is harder to learn a more concentrated preference without adequate data, the prediction performance of model (P1) and (P2) both become worser for a larger shape parameter, e.g., $\iota_i=1.2$.  
Hence from Fig. \ref{fig5} and Fig. \ref{fig6} (a), (b), we can conclude that $\hat{p}_i^2$, which is obtained by adaptively reweighting the learning samples from mobility patterns and transferring information across agents, can more accurately predict $p_i $ than $\hat{p}_i^1$ given by DRMTL model. This shows the feasibility of Algorithm 2 in real mobile applications.

  \begin{figure} 
	% 	\vskip 0.2in
	\begin{center}
		\includegraphics[width= 88 mm]{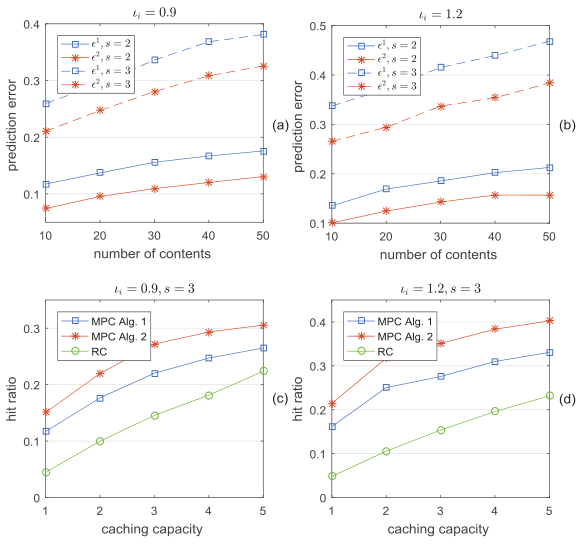} 
		\caption{(a), (b) geography preference prediction error with $\iota_i=\{ 0.9,1.2\}$; (c), (d) hit ratio of policies with $\iota_i=\{0.9,1.2\}$, $s=3$.}
		\label{fig6}
	\end{center}
	\vskip -0.2in
\end{figure}

% In the simulation setup, we let $F=20$, $\iota_i=\{0.9,1.2\}$, $t_d = 30$min, $s=\{2,3\}$ and $\mu_4 = 0.1$.
  
  Finally, we evaluate the average hit ratio defined in \cite{Chen2017TWC} for each agent, which represents the proportion of served requests by the caching of agents during time $[0, t_d]$. We assume the caching capacity for agents is $\theta (\in\mathbbm{N}_{+})$ and $\theta\leq F$. In proactive caching, each agent stores the contents at time $0$ according to the following two policies:
   \begin{enumerate}[leftmargin=*]
   	\item {\textit{Most popular caching}} (MPC) \cite{CMG16}: each agent proactively stores content according to descending order of its predicted preference until $\theta$ is occupied;
   	\item{\textit{Random caching}} (RC) \cite{CMG16}: each agent proactively stores content in random manner until $\theta$ is occupied.
   \end{enumerate} 
   For the MPC method, we consider two approaches, the Algorithms 1 and 2, to predict the content preference for agents. Since the stored contents are randomly selected for RC policy, there is no need to conduct the prediction of geography preference.
   In the simulation setup, we fix $F=20$ and $s=3$. The results on hit ratio are demonstrated in Fig. \ref{fig6} (c) and (d). It is seen that the hit ratio achieved by MPC with the geography preference predicted by mobility-aware DRMTL model can outperform the MPC with preference from Algorithm 1 and RC scheme. Since more requests can be served by agent locally, with enlarging the caching capacity $\theta$, the hit ratio for all methods grow. The increase of RC follows a constant rate, while the rate of MPC policies degrade. This is expected since the contents proactively cached are randomly chosen in RC policy. But the MPC will choose to store those contents that are more likely to be requested, and the growth rate of cumulative preference for the cached contents will reduce with increasing $\theta$, which is caused by the nature of Zip-f like distribution. It is known that the preference of MTs become more concentrated when the shape parameter $\iota_i$ gets larger. Hence with $\iota_i=1.2$, the hit ratios achieved by all policies are better than that of $\iota_i=0.9$. This does not violate the trends shown in Fig. \ref{fig6} (a) and (b), where prediction error degrades with $\iota_i$, since only the index of descending order of predicted preference counts for the hit ratio in proactive caching problem.

\section{Conclusions }

We study  mobility-aware preference prediction for decentralized networks with heterogeneous interests for the MTs. The DRMTL model is first formulated to tackle the learning problem with stationary MTs, which is solved by the proposed proximal ADMM based Algorithm 1 in a hybrid Jacobian and Gauss-Seidel type. Then we extend the DRMTL model by reweighting training samples and introducing a transfer penalty in objective. With the modification, the mobility-aware DRMTL model, which is consistent with the DRMTL model, can successfully capture the variation of geography preference when mobility patter of MTs is predictable. Our real trace driven simulation demonstrates that the MPC scheme with preference predicted by mobility-aware DRMTL model can achieve better performance on hit ratios than MPC with preference from DRMTL and RC scheme.

%We study the proactive content caching schemes for a two-tier UDN, considering the mobility prediction and content preference prediction. A discrete Markov renewal process is first utilized to acquire the mobility pattern of MTs. Then, we formulate the optimal content placements with the objectives of maximizing the hit-rate, and two caching scenarios, caching only at SBSs and caching at both SBSs and MTs, are considered, for which the corresponding optimal solution and heuristic algorithm MPPC are developed. Furthermore, we propose an ML-ELM, which is shown to guarantee a smaller training error than that of the single ELM learning approach, to predict the content preference. Through the real trace driven simulation, we validate the hit-rate gain of proposed MPPC through comparing with MPC and RC proactive caching schemes. The results demonstrate that the mobility-aware proactive caching scheme can achieve a better performance than MPC and RC schemes. But as the caching capacity of MTs and the shape parameter increase, the necessity of considering mobility in design of proactive caching scheme will reduce. Moreover, it is shown that the proposed ML-ELM can provide a reliable prediction on the content preference. 

 \begin{appendices}

	\section{Proof of Lemma 2}

	Define the intermediate variable as $\bm{z}^{k+\frac{1}{2}}_i=[\check{\bm{w}}^{k+1}_i,\hat{\bm{w}}^k_i]$. Then from three point inequality (\ref{eq8}) in Lemma 1, we have
	\begin{equation}\label{eq:6}
	\begin{aligned}
	&\vphantom{\frac{1}{2}} f_i \big( \bm{z}_i^{k+1} \big) - f_i \big( \bm{z}_t^*  \big) 
	\leq \big(\bm{z}_i^{k+1}-\bm{z}_i^* \big)^T\nabla f_i \big(\bm{z}_i^{k+\frac{1}{2}}\big) + \frac{C_i}{2}\big\|\bm{z}_i^{k+1} - \bm{z}_i^{k+\frac{1}{2}}  \big\|^2\\
	\vphantom{\frac{1}{2}} &= \big(\check{\bm{w}}_i^{k+1}-\check{\bm{w}}_i^* \big)^T\nabla_{\check{\bm{w}}_i} f_t \big( \bm{z}_i^{k+\frac{1}{2}} \big) + \big(\hat{\bm{w}}_i^{k+1}-\hat{\bm{w}}_i^* \big)^T\nabla_{\hat{\bm{w}}_i} f_i \big( \bm{z}_t^{k+\frac{1}{2}} \big) +\\&\vphantom{\frac{1}{2}}  \frac{C_i}{2} \big\|\hat{\bm{w}}_i^k-\hat{\bm{w}}_i^{k+1}  \big\|^2.
	\end{aligned}		
	\end{equation}
	Following (\ref{eq9}) and (\ref{eq10}) in Lemma 1, it is easy to show that
	\begin{equation}\label{eq:7} 
	\begin{aligned}
	\vphantom{\frac{1}{2}}g \big(\check{\bm{w}}_i^{k+1}\big)-g\big(\check{\bm{w}}_i^* \big) \leq&  \big(\check{\bm{w}}_i^{k+1} - \check{\bm{w}}_i^*\big)^Tg'\big(\check{\bm{w}}_i^{k+1} \big),\\   
	h \big(\hat{\bm{w}}_i^{k+1} \big)-h\big(\hat{\bm{w}}_i^* \big) \leq&  \big(\hat{\bm{w}}_i^{k+1} - \hat{\bm{w}}_i^*\big)^Th'\big(\hat{\bm{w}}_i^{k+1}\big)-  \frac{m}{2} \big\|\hat{\bm{w}}_i^{k+1}-\hat{\bm{w}}_i^* \big\|^2.
	\end{aligned}
	\end{equation}
	Meanwhile the KKT conditions are
	\begin{equation}\label{eq:8} 
	\bm{A}\check{\bm{W}}^* = 0, ~\text{and}~A^T_t\lambda^*=\nabla_{\check{\bm{w}}_i}F_i\big(\bm{z}_i^* \big),~\bm{0}=\nabla_{\hat{\bm{w}}_i}F_i\big(\bm{z}_i^* \big).
	% ~\nabla F_t\left(Z_t^*\right)=\begin{bmatrix}
	%\nabla_{U_t}F_t\left(Z_t^*\right) \\\nabla_{A_t}F_t\left(Z_t^*\right)
	%\end{bmatrix}=
	%\begin{bmatrix}
	%A^T_t\lambda^* \\\bm{0}
	%\end{bmatrix}.
	\end{equation}
	Furthermore, considering the convexity of $F_i\big(\bm{z}_i\big)$ as (\ref{eq9}), we have
	\begin{equation}\label{eq:9}
	\begin{aligned}
	F_i \big(\bm{z}_i^{k+1} \big) - F_i \big( \bm{z}_i^* \big)\geq \big(\bm{z}_i^{k+1} -\bm{z}_i^* \big)^T \nabla F_i\big(\bm{z}_i^*\big) 
	= \big(\check{\bm{w}}_i^{k+1}-\check{\bm{w}}_i^* \big)^TA_i^T\lambda^*.
	\end{aligned}	
	\end{equation}
	Combining (\ref{eq:6}), (\ref{eq:7}), (\ref{eq:8}) and (\ref{eq:9}), we can obtain
	\begin{equation}\label{eq:10}
	\begin{aligned}
	0\leq& \vphantom{\frac{1}{2}}  \big( \check{\bm{w}}_i^{k+1} - \check{\bm{w}}_i^*  \big)^T \big( \nabla_{\check{\bm{w}}_i} F_i \big(\bm{z}_i^{k+\frac{1}{2}}  \big)-A_t^T\lambda^*  \big) +  \big(\hat{\bm{w}}_i^{k+1} - \hat{\bm{w}}_i^* \big)^T\cdot\\
	&\vphantom{\frac{1}{2}}  \big( \nabla_{\hat{\bm{w}}_i}F_i \big(\bm{z}_i^{k+1}  \big) + \nabla_{\hat{\bm{w}}_i}f_i \big( \bm{z}_i^{k+\frac{1}{2}}  \big) -\nabla_{\hat{\bm{w}}_i}f_i \big(\bm{z}_i^{k+1}  \big) \big)+\\
	&\vphantom{\frac{1}{2}} \frac{C_i}{2} \big\|\hat{\bm{w}}_i^k- \hat{\bm{w}}_i^{k+1}  \big\|^2 -\frac{m}{2} \big\|\hat{\bm{w}}_i^{k+1}-\hat{\bm{w}}_i^*  \big\|^2\\
	\leq &\vphantom{\frac{1}{2}}  \big( \check{\bm{w}}_i^{k+1} - \check{\bm{w}}_i^*  \big)^T \big(   \nabla_{\check{\bm{w}}_i} F_i \big(\bm{z}_i^{k+\frac{1}{2}}  \big)-A_t^T\lambda^*  \big) +  \big(\hat{\bm{w}}_i^{k+1} - \hat{\bm{w}}_i^* \big)^T\cdot\\
	&\vphantom{\frac{1}{2}}  \nabla_{\hat{\bm{w}}_i}F_i \big(\bm{z}_i^{k+1}  \big)+C_i \big\| \hat{\bm{w}}_i^{k+1}-\hat{\bm{w}}_i^*  \big\|  \big\|\hat{\bm{w}}_i^{k} -\hat{\bm{w}}_i^{k+1} \big\|+\\
	& \frac{C_i}{2} \big\|\hat{\bm{w}}_i^k- \hat{\bm{w}}_i^{k+1}  \big\|^2 -\frac{m}{2} \big\|\hat{\bm{w}}_i^{k+1}-\hat{\bm{w}}_i^*  \big\|^2\\
	\overset{(a)}{\leq} & \vphantom{\frac{1}{2}} \big( \check{\bm{w}}_i^{k+1} - \check{\bm{w}}_i^*  \big)^T \big(   \nabla_{\check{\bm{w}}_i} F_i \big(\bm{z}_i^{k+\frac{1}{2}}  \big)-A_i^T\lambda^*  \big) + \big(\hat{\bm{w}}_i^{k+1} - \hat{\bm{w}}_i^* \big)^T\cdot\\
	& \nabla_{\hat{\bm{w}}_i}F_i \big(\bm{z}_i^{k+1} \big)+\frac{C_i}{2m} \big(C_i+m \big)  \big\|\hat{\bm{w}}_i^{k} -\hat{\bm{w}}_i^{k+1} \big\|^2,				
	\end{aligned}
	\end{equation}
	where $(a)$ is the inequality $C_i\| \hat{\bm{w}}_i^{k+1}-\hat{\bm{w}}_i^* \|  \|\hat{\bm{w}}_i^{k} -\hat{\bm{w}}_i^{k+1} \|\leq \frac{m}{2}\|\hat{\bm{w}}_i^{k+1}-\hat{\bm{w}}_i^* \|^2 + \frac{C_i^2}{2m}\|\hat{\bm{w}}_i^k-\hat{\bm{w}}_i^{k+1} \| ^2$. Recall the $k$-th iteration in Algorithm 1, agent $i$ first solves (\ref{eq3}), with the optimality conditions as
	%\begin{equation}\label{eq:11}
	%\begin{aligned}
	%\check{\bm{w}}_i^{k+1}:=\arg&\min_{\check{\bm{w}}_i}~f_i \big(\check{\bm{w}}_i,\hat{\bm{w}}_i^k \big) + g_1 \big(\check{\bm{w}}_i \big) + \lambda^kA_i\check{\bm{w}}_i\\
	%&\frac{\rho}{2} \big\|A_i\check{\bm{w}}_i + \sum\nolimits_{j\neq i}A_j\check{\bm{w}}_j^k   \big\|^2+ \frac{1}{2} \big\|\check{\bm{w}}_i-\check{\bm{w}}_i^k   \big\|^2_{P_i}. 
	%\end{aligned}	 
	%\end{equation}
	%The optimality condition of (\ref{eq:11}) is given by
	\begin{equation}\label{eq:12}
	\begin{aligned}
	&\vphantom{\frac{1}{2}}  A_i^T \big[  \lambda^k -\rho \big( A_i\check{\bm{w}}_i^{k+1} + \sum\nolimits_{j\neq i}A_j\check{\bm{w}}_j^k  \big)  \big]+P_i \big(\check{\bm{w}}_i^k-\check{\bm{w}}_i^{k+1}  \big)\\
	=&\vphantom{\frac{1}{2}}A_i^T \big[  \lambda^k + \rho A_i \big(\check{\bm{w}}_i^k-\check{\bm{w}}_i^{k+1} \big) -\rho \bm{A} \big(\check{\bm{W}}^k - \check{\bm{W}}^{k+1} \big) -\big.\\		
	&\vphantom{\frac{1}{2}}\big.\rho \bm{A} \big(\check{\bm{W}}^{k+1}-\check{\bm{W}}^* \big) \big] + P_i \big(\check{\bm{w}}_i^k-\check{\bm{w}}_i^{k+1}  \big)= \nabla_{\check{\bm{w}}_i}F_i \big(\bm{z}_i^{k+\frac{1}{2}} \big).
	\end{aligned}		
	\end{equation}
	Then the agent $i$ updates $\hat{\bm{w}}_i$ by solving the subproblem (\ref{eq5}). The optimality condition gives
	%\begin{equation}\label{eq:13}
	%\hat{\bm{w}}_i^{k+1}=\arg\min_{\hat{\bm{w}}_i}~f_t \big(\check{\bm{w}}_i^{k+1},\hat{\bm{w}}_i \big) + g_2 \big(\hat{\bm{w}}_i \big)  + \frac{1}{2} \big\|\hat{\bm{w}}_i-\hat{\bm{w}}_i^k  \big\|^2_{Q_t}. 
	%\end{equation}
	%The corresponding optimality condition of (\ref{eq:13}) is
	\begin{equation}\label{eq:14}
	Q_i \big( \hat{\bm{w}}_i^k-\hat{\bm{w}}_i^{k+1} \big)=\nabla_{\hat{\bm{w}}_i}F_i\big(\bm{z}_i^{k+1} \big)  .
	\end{equation}
	Then by plugging (\ref{eq:12}) and (\ref{eq:14}) into (\ref{eq:10}), we have following inequality
	\begin{equation}\label{eq:15}
	\begin{aligned}
	& \vphantom{\frac{1}{2}}  \big(\check{\bm{w}}_i^{k+1} - \check{\bm{w}}_i^*  \big)^TA_i^T\big(\lambda^k-\lambda^*  \big) +  \big(\check{\bm{w}}_i^{k+1} - \check{\bm{w}}_i^* \big)^T \big(\rho A_i^TA_i + P_i \big) \cdot\\
	& \vphantom{\frac{1}{2}}\big( \check{\bm{w}}_i^k-\check{\bm{w}}_i^{k+1} \big) + \big(\hat{\bm{w}}_i^{k+1} - \hat{\bm{w}}_i^* \big)^T Q_i \big( \hat{\bm{w}}_i^k-\hat{\bm{w}}_i^{k+1} \big)\\
	\geq &\rho  \big(\check{\bm{w}}_i^{k+1} - \check{\bm{w}}_i^*\big)^T A_i^T \bm{A} \big(\check{\bm{W}}^k -\check{\bm{W}}^{k+1} \big)    - \frac{C_i}{2m} \big(C_i+m \big)\cdot\\  
	&\vphantom{\frac{1}{2}}\big\|\hat{\bm{w}}_i^k- \hat{\bm{w}}_i^{k+1} \big\|^2+ \rho   \big(\check{\bm{w}}_i^{k+1} - \check{\bm{w}}_i^*\big)^TA_i^T\bm{A} \big(\check{\bm{W}}^{k+1} -\check{\bm{W}}^* \big). 
	\end{aligned}
	\end{equation}
	Since $ \lambda^{k+1} = \lambda^k -\gamma \rho \bm{A}\check{\bm{W}}^{k+1}$, it can be derived that 
	\begin{equation}
	\bm{A} \big(\check{\bm{W}}^{k+1} -\check{\bm{W}}^* \big) = \frac{1}{\gamma \rho }  \big(\lambda^k-\lambda^{k+1}  \big).
	\end{equation}
	Noting the truth that $\lambda^k-\lambda^* = \lambda^k-\lambda^{k+1} + \lambda^{k+1} - \lambda^*$, and summing the inequality (\ref{eq:15}) over all $i\in\mathcal{V}$, we obtain
	\begin{equation}
	\begin{aligned}
	&\vphantom{\frac{1}{2}}\frac{1}{\gamma \rho } \big( \lambda^k- \lambda^{k+1}\big)^T\big( \lambda^k-\lambda^* \big) + \sum\nolimits_{i} \big[  \big(\check{\bm{w}}_i^{k+1} - \check{\bm{w}}_i^* \big)^T \big(\rho A_i^TA_i + P_i \big)\cdot \\
	&\vphantom{\frac{1}{2}} \big( \check{\bm{w}}_i^k-\check{\bm{w}}_i^{k+1} \big) + \big(\hat{\bm{w}}_i^{k+1} - \hat{\bm{w}}_i^*  \big)^T  Q_i \big( \hat{\bm{w}}_i^k-\hat{\bm{w}}_i^{k+1}  \big) \big]\\
	\geq &\vphantom{\frac{1}{2}}\frac{1-\gamma}{\gamma^2\rho}  \big\|\lambda^k - \lambda^{k+1}  \big\|^2 + \frac{1}{\gamma } \big(  \lambda^k-\lambda^{k+1}\big)^T \bm{A} \big( \check{\bm{W}}^k-\check{\bm{W}}^{k+1} \big) -\\
	&\vphantom{\frac{1}{2}}\sum\nolimits_{i} \frac{C_i}{2m} \big(C_i+m \big) \big\| \hat{\bm{w}}_i^k - \hat{\bm{w}}_i^{k+1}  \big\|^2,
	\end{aligned}
	\end{equation}
	or more compactly,
	\begin{equation}
	\begin{aligned}
	&\vphantom{\frac{1}{2}}\big(\bm{u}^k-\bm{u} ^{k+1} \big)^T\bm{G} \big(\bm{u}^{k+1}-\bm{u}^*\big) 
	\geq \vphantom{\frac{1}{2}} \frac{1-\gamma}{\gamma^2\rho} \big\|\lambda^k - \lambda^{k+1} \big\|^2 +\\& \frac{1}{\gamma } \big(  \lambda^k-\lambda^{k+1}\big)^T  \bm{A} \big(\check{\bm{W}}^k-\check{\bm{W}}^{k+1} \big)- \vphantom{\frac{1}{2}}\sum\nolimits_{i}\frac{C_i}{2m} \big(C_i+m \big) \big\| \hat{\bm{w}}_i^k - \hat{\bm{w}}_i^{k+1}  \big\|^2.
	\end{aligned}
	\end{equation}
	With the equality $\| \bm{u}^k - \bm{u}^* \|^2_{\bm{G}} - \| \bm{u}^{k+1} -\bm{u}^*\|^2_{\bm{G}} = 2 (\bm{u}^k-\bm{u} ^{k+1}  )^T{\bm{G}} (\bm{u}^{k+1}-\bm{u} ^* ) + \|\bm{u}^k-\bm{u}^{k+1} \|^2_{\bm{G}}$, we can get the desired result as 
	\begin{equation}
	\begin{aligned} 	
	& \vphantom{\frac{1}{2}}\big\| \bm{u}^k - \bm{u}^* \big\|^2_{\bm{G}} - \big\| \bm{u}^{k+1} -\bm{u}^*\big\|^2_{\bm{G}} 
	\geq \frac{2-\gamma}{\gamma^2\rho }  \big\|\lambda^k-\lambda^{k+1}  \big\|^2 +\\& \vphantom{\frac{1}{2}}\frac{2}{\gamma} \big(  \lambda^k-\lambda^{k+1}\big)^T  \bm{A} \big(\check{\bm{W}}^k-\check{\bm{W}}^{k+1} \big) +  
	\big\|\check{\bm{W}}^k-\check{\bm{W}}^{k+1}  \big\|^2_{\bm{G}_1} +\\  &\big\|\hat{\bm{W}}^k-\hat{\bm{W}}^{k+1}  \big\|^2_{\bm{G}_2-\bm{G}_3}
	\vphantom{\frac{1}{2}}=\big\|\bm{u}^k-\bm{u}^{k+1}  \big\|^2_{\bm{M}}.
	\end{aligned} 
	\end{equation}
	
	\section{proof of Theorem 1}
	To prove the convergence of Alg.1, we should ensure that $\|\bm{u} \|_{\bm{M}}^2\geq0$ holds for any $\bm{u}$, which is equivalent to guaranteeing $\bm{M}$ to be semi-positive. We have
	\begin{equation}\label{eq:17}
	\big\|\bm{u} \big\|_{\bm{M}}^2 = \big\|\check{\bm{W}}  \big\|_{\bm{G}_1}^2 +  \big\| \hat{\bm{W}}  \big\|^2_{\bm{G}_2-\bm{G}_3} + \frac{2-\gamma}{\gamma^2\rho}\big\|\lambda \big\|^2+ \frac{2}{\gamma} \lambda^T\bm{A}\check{\bm{W}}.	
	\end{equation}
	Moreover, we can derive the following inequality  
	\begin{equation}\label{eq:18}
	\begin{aligned}
	\frac{2}{\gamma} \lambda^T\bm{A}\check{\bm{W}}= \sum\nolimits_{i}\lambda^TA_i\check{\bm{w}}_i 
	\geq  -\sum\nolimits_{i} \bigg( \frac{\epsilon_i}{\gamma^2\rho}\big\|\lambda\big\|^2 + \frac{\rho}{\epsilon_i}  \big\| A_i\check{\bm{w}}_i \big\|^2  \bigg)^2,
	\end{aligned}	
	\end{equation}
	where $\epsilon_i>0$. Hence putting (\ref{eq:18}) into (\ref{eq:17}), we can obtain
	\begin{equation}
	\begin{aligned}
	\big\|\bm{u} \big\|^2_M \geq& \vphantom{\frac{1}{2}}\sum\nolimits_{i} \bigg( \big\| \check{\bm{w}}_i \big\|_{P_i+\rho (1 -\frac{1}{\epsilon_i})A_i^TA_i}^2+ \big\|\hat{\bm{w}}_i \big\|_{Q_i-\frac{C_i}{m}(C_i+m)I}^2 \bigg)+\\
	& \vphantom{\frac{1}{2}} \frac{2-\gamma -\sum\nolimits_{i}\epsilon_i}{\gamma^2\rho}\big\| \lambda\big\|^2.
	\end{aligned}
	\end{equation}
	Since $A_i^TA_i=d_iI$, guaranteeing that $P_i+\rho (1 -\frac{1}{\epsilon_i})d_iI\succ 0$, $Q_i-\frac{C_i}{m}(C_i+m )I\succ 0$ and $2-\gamma -\sum\nolimits_{i}\epsilon_i>0$, makes $\bm{M}$ positive definite and $ \| \bm{u} \|^2_M\geq0$. Hence the sequence $\bm{u}^k$ converges to $\bm{u}^*$ as $k\to\infty$ according to \cite{Deng2017}. Since $P_i=\tau_iI$ and $Q_i=\zeta_iI$, we get the final results as (\ref{eq19}).
	%$$ \tau_i>\rho ( \frac{1}{\epsilon_i}-1)d_i,~\text{and} ~\zeta_i>  $$ 
	
	\section{Proof of Corollary 1} 
	
	Combining (\ref{eq:6}), (\ref{eq:7}), (\ref{eq:8}), (\ref{eq:12}), (\ref{eq:14}) and summing over all $i$, we can get
	\begin{equation}\label{eq:19}
	\begin{aligned}
	&\vphantom{\frac{1}{2}}\bm{F} \big(\bm{Z}^{k+1} \big) -  \bm{F}\big(\bm{Z}^* \big)  \leq  \sum\nolimits_{i} \bigg[   \big( \check{\bm{w}}_i^{k+1} - \check{\bm{w}}_i^* \big)^T  \nabla_{\check{\bm{w}}_i} F_i \big(\bm{z}_i^{k+\frac{1}{2}} \big) +   \\&\vphantom{\frac{1}{2}}   \big(\hat{\bm{w}}_i^{k+1} - \hat{\bm{w}}_i^* \big)^T \nabla_{\hat{\bm{w}}_i}F_i \big(\bm{z}_i^{k+1}  \big)+\frac{C_i}{2m} \big(C_i+m \big)  \big\|\hat{\bm{w}}_i^{k} - \hat{\bm{w}}_i^{k+1}  \big\|^2 \bigg]\\
	%		=&\sum_{t=1}^{m}\big(U_t^{k+1}-U_t^*\big\big)^T\big(\rho C_t^TC_t + P_t \big\big)\big(U_t^k-U_t^{k+1} \big\big)-\big\langle C\big( \overline{U}^{k+1} -\overline{U}^*\big\big),\rho C\big(\overline{U}^k-\overline{U}^{k+1} \big\big) \big\rangle-\\
	=& \vphantom{\frac{1}{2}} \big( \check{\bm{W}}^{k+1}-\check{\bm{W}}^* \big)^T \big(\bm{G}_1-\rho \bm{A}^T\bm{A} \big) \big(\check{\bm{W}}^k-\check{\bm{W}}^{k+1} \big)+ \\
	&\big(\hat{\bm{W}}^{k+1}-\hat{\bm{W}}^* \big)^T\bm{G}_2 \big(\hat{\bm{W}} ^k-\hat{\bm{W}}^{k+1} \big)+\frac{1}{2} \big\|\hat{\bm{W}}^k-\hat{\bm{W}}^{k+1}  \big\|_{\bm{G}_3}^2 +\\
	& \vphantom{\frac{1}{2}}\frac{1}{\gamma \rho} \big( \lambda^k-\lambda^{k+1}\big)^T\lambda^k   -\frac{1}{\gamma^2 \rho} \big\|\lambda^k-\lambda^{k+1} \big\|^2.
	\end{aligned}
	\end{equation}
	With the condition $ \bm{G}_1^{\dagger}=\bm{G}_1-\rho \bm{A}^T\bm{A} \succ 0$, we have
	\begin{equation}\label{eq:20}
	\begin{aligned}
	\big\|\check{\bm{W}}^k-\check{\bm{W}}^{k+1} \big\|_{ \bm{G }_1^{\dagger}}^2\geq 0.
	\end{aligned}
	\end{equation}
	While with $\bm{G}_2\succ \bm{G}_3$, we can get
	\begin{equation}\label{eq:21}
	\begin{aligned}
	\big\|\hat{\bm{W}}^k-\hat{\bm{W}}^{k+1}  \big\|_{\bm{G}_2}^2 \geq  \big\|\hat{\bm{W}}^k-\hat{\bm{W}}^{k+1} \big\|_{\bm{G}_3}^2.
	\end{aligned}
	\end{equation}
	Moreover, for the last two terms in (\ref{eq:19}), we can derive
	\begin{equation}\label{eq:22}
	\begin{aligned}
	&\frac{1}{\gamma \rho} \big(  \lambda^k-\lambda^{k+1}\big)^T\lambda^k  -\frac{1}{\gamma^2\rho}  \big\|\lambda^k-\lambda^{k+1} \big\|^2\\
	= &\frac{1}{2\gamma \rho} \big\|\lambda^k-\lambda^{k+1} \big\|^2 +\frac{1}{2\gamma \rho} \big(   \lambda^k-\lambda^{k+1}\big)^T\big( \lambda^{k }+\lambda^{k+1}  \big)-\\ & \frac{1}{\gamma^2\rho} \big\|\lambda^k-\lambda^{k+1} \big\|^2 
\overset{(a)}{\leq}  \frac{1}{2\gamma \rho} \big( \big\|\lambda ^k \big\|^2 -  \big\|\lambda^{k+1} \big\|^2  \big), 
	\end{aligned}
	\end{equation}
	where $(a)$ is because $0<\gamma<2$. The following result can be obtained by integrating (\ref{eq:19}), (\ref{eq:20}), (\ref{eq:21}) and (\ref{eq:22}) as
	\begin{equation}
	\begin{aligned}
	& \vphantom{\frac{1}{2}} \bm{F} \big(\bm{Z}^{k+1} \big) -  \bm{F}\big(\bm{Z}^* \big) \leq \big(\bm{Z}^{k}-\bm{Z}^{k+1} \big)^T\bm{G}_{12}^{\dagger}  \big(\bm{Z}^{k+1}-\bm{Z}^* \big) +  \\
	& \frac{1}{2}  \big\|\bm{Z}^{k}-\bm{Z}^{k+1} \big\|_{\bm{G}_{12}^{\dagger} }^2 +\frac{1}{2\gamma \rho} \big( \big\|\lambda ^{k} \big\|^2 - \big\|\lambda^{k+1}  \big\|^2   \big) \\
	=&\frac{1}{2} \bigg[ \big\|\bm{Z}^{k}-\bm{Z}^*  \big\|^2_{\bm{G}_{12}^{\dagger} } -  \big\| \bm{Z}^{k+1}-\bm{Z}^*  \big\|^2_{\bm{G}_{12}^{\dagger} }  +  \frac{1}{ \gamma \rho} \big( \big\|\lambda ^{k}  \big\|^2 -  \big\|\lambda^{k+1} \big\|^2 \big)  \bigg],
	\end{aligned}
	\end{equation}
	where $\bm{G}_{12}^{\dagger}:=\text{blkdiag}(\bm{G}_1^{\dagger},\bm{G}_2)$. Summing the above inequality over $i=1,...,k$, we get 
	\begin{equation}
	\begin{aligned}
	&  \bm{F} \big(\overline{\bm{Z}}^{k} \big) - \bm{F} \big(\bm{Z}^* \big) 
	\overset{(a)}{\leq} \frac{1}{k} \sum\nolimits_{j=1}^{k} \bm{F}\big(\bm{Z}^{j } \big) - \bm{F} \big(\bm{Z}^* \big)\\
	%			\leq &\frac{1}{2K}\big\big(\big\|\overline{U}^0-\overline{U}^* \big\|^2_{\overline{G}_U} + \big\|A^0-A^* \big\|^2_{G_A} +\frac{1}{\gamma \rho}\big\|\lambda^0 \big\|^2 \big),
	&~~~~ \leq \frac{1}{2k} \bigg(  \big\| \check{\bm{W}}^0-\check{\bm{W}}^*  \big\|^2_{\bm{G}_1^{\dagger}} + \big\|\hat{\bm{W}}^0-\hat{\bm{W}}^*  \big\|^2_{\bm{G}_2} +\frac{1}{\gamma \rho} \|\lambda^0  \|^2  \bigg),
	\end{aligned}
	\end{equation}
	where $(a)$ is from the convexity of objective $\bm{F} $. Letting $\lambda^0=\bm{0}$ completes the proof.
	
	\section{Proof of Proposition 2} 
	
	Considering the setting up for $A_i$, it can be deduced that $A_i^TA_j=-I_{n\times n}$ if $(i,j)\in\mathcal{E}$ otherwise $A_i^TA_j=\bm{0}$. Hence we can conclude $\|A_i^TA_j \| \leq \sqrt{n}$.
	Since $\bm{G}_1^{\dagger} \succ 0$ and $\bm{G}_2\succ \bm{G}_3$ are required from the proof of Corollary 1, for any $\check{\bm{W}}$ we need guarantee that
	\begin{equation}\label{eq59}
	\begin{aligned}
	\vphantom{\frac{1}{2}}\big\|\check{\bm{W}}  \big\|^2_{\bm{G}_1^{\dagger}}
	%		  =&\text{tr} (\bm{U}^TG_{\bm{U}}\bm{U} )-\rho \text{tr} (\bm{U}^T\bm{C}^T\bm{C}\bm{U} )\\
	= \vphantom{\frac{1}{2}}&\sum\nolimits_{i} \big( \check{\bm{w}}_i^TP_i\check{\bm{w}}_i  - \rho \sum\nolimits_{j}  \check{\bm{w}}_i^TA_i^TA_j\check{\bm{w}}_j\big)  \\
	= \vphantom{\frac{1}{2}}&\sum\nolimits_{i} \big[\check{\bm{w}}_i^T  \big(P_i-\rho A_i^TA_i \big) \check{\bm{w}}_i - \rho \sum\nolimits_{j\neq i} \check{\bm{w}}_i^TA_i^TA_j\check{\bm{w}}_j \big] \\
	\overset{(a)}{\geq}  \vphantom{\frac{1}{2}}& \sum\nolimits_{i}  \big[\check{\bm{w}}_i^T \big(P_i-\rho A_i^TA_i \big) \check{\bm{w}}_i   - \rho \sqrt{n} \sum\nolimits_{j\neq i} \big\| \check{\bm{w}}_i \big\| \big\|\check{\bm{w}}_j \big\|\big] \\
	\geq \vphantom{\frac{1}{2}}& \sum\nolimits_{i}   \big\| \check{\bm{w}}_i \big\|^2_{P_i-\rho A_i^TA_i-4(N-1)\rho\sqrt{n} I}\geq 0,
	\end{aligned}
	\end{equation}
	where $(a)$ is because $  \|A_i^TA_j  \| \leq \sqrt{n}(i\neq j)$. The inequality (\ref{eq59}) can be satisfied if $P_i-\rho A_i^TA_i-4(N-1)\rho \sqrt{n} I\succ 0$. With $P_i = \tau_iI$ and $A_i^TA_i = d_iI$, it reduces to $\tau_i>\rho d_i + 4(N-1)\rho\sqrt{n}$. The condition for $Q_i$ required in Corollary 1 are consistent with that in Theorem 1. Combining the results in Remark 1 completes the proof.

	\section{Proof of Lemma 3} 
	From Assumption 3 we know $\tilde{f}_i^j$ is jointly convex over $\check{\bm{w}}_i$ and $\hat{\bm{w}}_i$. Thus with the convexity of term $\|\bm{w}_{j,i}^{\text{loc}}-(\check{\bm{w}}_i+\hat{\bm{w}}_i ) \|^2$ in (\ref{eq35}), we can conclude that $\mathbbm{f}_i$ is differentiable and convex. The gradient of $\nabla \tilde{f}_i$ satisfies
	\begin{equation}
	\big\| \nabla \tilde{f}_i\big(\bm{z}_i^1 \big)- \nabla \tilde{f}_i\big(\bm{z}_i^2  \big) \big\|\leq \tilde{C}_i\big\|   \bm{z}_i^1-\bm{z}_i^2  \big\|,~\bm{z}_i^1,\bm{z}_i^2\in\mathbbm{R}^n\times \mathbbm{R}^n. 
	\end{equation}
	Considering the partial gradient of $\mathbbm{f}_i$ over $\check{\bm{w}}_i$, we can derive the following inequality.
	\begin{equation}
	\begin{aligned}
	&\vphantom{\frac{1}{2}}\big\| \partial \mathbbm{f}_i\big(\bm{z}_i^1 \big)- \partial \mathbbm{f}_i\big(\bm{z}_i^2  \big) \big\|^2=\big\| \partial \tilde{f}_i\big(\bm{z}_i^1\big)- \partial \tilde{f}_i\big(\bm{z}_i^2\big)+\mu_3\sum\nolimits_{j\in\mathcal{V}_i}c_{j,i}\cdot\\
	&\vphantom{\frac{1}{2}} \big(\check{\bm{w}}_i^1-\check{\bm{w}}_i^2 + \hat{\bm{w}}_i^1-\hat{\bm{w}}_i^2 \big)  \big\|^2\leq 2\big\| \partial \tilde{f}_i\big(\bm{z}_i^1\big)- \partial \tilde{f}_i\big(\bm{z}_i^2\big)\big\|^2+ 2\mu_3^2\cdot\\
	&\vphantom{\frac{1}{2}}\left(\sum\nolimits_{j\in\mathcal{V}_i }c_{j,i}\right)^2\big\| \check{\bm{w}}_i^1-\check{\bm{w}}_i^2 + \hat{\bm{w}}_i^1-\hat{\bm{w}}_i^2 \big\|^2.\label{eq62}
	\end{aligned}
	\end{equation}
	Since $\| \check{\bm{w}}_i^1-\check{\bm{w}}_i^2 + \hat{\bm{w}}_i^1-\hat{\bm{w}}_i^2 \|^2\leq 2(\|\check{\bm{w}}_i^1-\check{\bm{w}}_i^2 \|^2+\|\hat{\bm{w}}_i^1-\hat{\bm{w}}_i^2 \|^2)=2\|\bm{z}_i^1-\bm{z}_i^2 \|^2$, summing the inequality (\ref{eq62}) over partial gradient over $\check{\bm{w}}_i$ and $\hat{\bm{w}}_i$, we obtain 
	\begin{equation}
	\begin{aligned}
	&\vphantom{\frac{1}{2}}\big\| \nabla \mathbbm{f}_i\big(\bm{z}_i^1 \big)- \nabla \mathbbm{f}_i\big(\bm{z}_i^2  \big) \big\|^2\leq 2\big\|  \nabla \tilde{f}_i\big(\bm{z}_i^1 \big)- \nabla \tilde{f}_i\big(\bm{z}_i^2  \big)\big\|^2 + 4\mu_3^2\cdot\\
	&\vphantom{\frac{1}{2}}\left(\sum\nolimits_{j\in\mathcal{V}_i }c_{j,i}\right)^2\big\|\bm{z}_i^1-\bm{z}_i^2 \big\|^2\\
	\leq &\bigg[ 2\tilde{C}_i^2 + 4\mu_3^2\left(\sum\nolimits_{j\in\mathcal{V}_i }c_{j,i}\right)^2 \bigg]\big\|\bm{z}_i^1-\bm{z}_i^2 \big\|^2.
	\end{aligned}
	\end{equation}
	We let $\mathbbm{C}_i=\sqrt{2\tilde{C}_i^2 + 4\mu_3^2 (\sum\nolimits_{j\in\mathcal{V}_i }c_{j,i} )^2}$ to make (\ref{eq7}) satisfied for $\mathbbm{f}_i$. Then following the proof of Lemma 1, it can be shown that (\ref{eq8}) also holds for $\mathbbm{f}_i$ but replacing $C_i$ with $\mathbbm{C}_i$.
\end{appendices}

 \bibliography{ref1}
 \bibliographystyle{IEEEtran}

\end{document}